\documentclass[letterpaper, 10 pt, conference]{ieeeconf}  

\IEEEoverridecommandlockouts                              

\usepackage{graphics} 
\usepackage{epsfig} 
\usepackage{times} 
\usepackage{amsmath} 
\usepackage{amssymb}  
\usepackage[ruled,vlined,linesnumbered]{algorithm2e}
\SetAlgoSkip{}
\usepackage{hyperref}
\usepackage{bm}
\usepackage{gensymb}
\usepackage{caption}
\usepackage{subcaption}
\usepackage{xcolor}
\usepackage{svg}
\usepackage{wrapfig}
\usepackage{algpseudocode}
\usepackage{booktabs}
\usepackage{threeparttable}
\usepackage{array} 
\newcolumntype{C}[1]{>{\centering\arraybackslash}p{#1}} 
\newtheorem{theorem}{Theorem}

\newtheorem{lemma}[theorem]{Lemma}
\newtheorem{definition}{Definition}
\newtheorem{problem}{Problem}
\DeclareMathOperator{\interior}{int}

\newcommand{\R}{\mathbb{R}}

\newcommand{\by}{\mathbf{y}}
\newcommand{\bx}{\mathbf{x}}
\newcommand{\bz}{\mathbf{z}}
\newcommand{\bW}{\mathbf{W}}
\newcommand{\bb}{\mathbf{b}}

\newcommand{\bTheta}{\mathbf{\Theta}}
\newcommand{\blambda}{\mathbf{\lambda}}
\newcommand{\bLambda}{\mathbf{\Lambda}}
\newcommand{\ba}{\mathbf{a}}
\newcommand{\bA}{\mathbf{A}}
\newcommand{\bC}{\mathbf{C}}
\newcommand{\bd}{\mathbf{d}}
\newcommand{\bdelta}{\mathbf{\delta}}
\newcommand\scalemath[2]{\scalebox{#1}{\mbox{\ensuremath{\displaystyle #2}}}}

\newenvironment{rcases}
  {\left.\begin{aligned}}
  {\end{aligned}\right\rbrace}

\title{\LARGE \bf
Reachable Polyhedral Marching (RPM): An Exact Analysis Tool for Deep-Learned Control Systems}

\author{Joseph A. Vincent$^{1}$ and Mac Schwager$^{1}$
\thanks{$^{1}$Department of Aeronautics and Astronautics, Stanford University, Stanford, CA 94305, USA, {\texttt\small \{josephav, schwager\}@stanford.edu}}
\thanks{The first author was supported in part by a Dwight D. Eisenhower Transportation Fellowship. The NASA University Leadership Initiative (grant \#80NSSC20M0163) provided funds to assist the authors with their research, but this article solely reflects the opinions and conclusions of its authors and not any NASA entity. We are grateful for this support.}%
\thanks{Code: \href{https://github.com/StanfordMSL/Neural-Network-Reach}{\text{https://github.com/StanfordMSL/Neural-Network-Reach}}.}
}

\begin{document}

\maketitle
\thispagestyle{plain}
\pagestyle{plain}

\begin{abstract} \label{Abstract}
Neural networks are increasingly used in robotics as policies, state transition models, state estimation models, or all of the above.
With these components being learned from data, it is important to be able to analyze what behaviors were learned and how this affects closed-loop performance.
In this paper we take steps toward this goal by developing methods for computing control invariant sets and regions of attraction (ROAs) of dynamical systems represented as neural networks.
We focus our attention on feedforward neural networks with the rectified linear unit (ReLU) activation, which are known to implement continuous piecewise-affine (PWA) functions.
We describe the Reachable Polyhedral Marching (RPM) algorithm for enumerating the affine pieces of a neural network through an incremental connected walk.
We then use this algorithm to compute exact forward and backward reachable sets, from which we provide methods for computing control invariant sets and ROAs.
Our approach is unique in that we find these sets incrementally, without Lyapunov-based tools.
In our examples we demonstrate the ability of our approach to find non-convex control invariant sets and ROAs on tasks with learned van der Pol oscillator and pendulum models.
Further, we provide an accelerated algorithm for computing ROAs that leverages the incremental and connected enumeration of affine regions that RPM provides.
We show this acceleration to lead to a 15x speedup in our examples.
Finally, we apply our methods to find a set of states that are stabilized by an image-based controller for an aircraft runway control problem.
\end{abstract}

\section{Introduction}
\label{Sec:Introduction}
\subsection{Overview}

In this paper we describe the Reachable Polyhedral Marching (RPM) algorithm for enumerating the affine regions of neural networks with rectified linear unit (ReLU) activation.
This algorithm can then be leveraged to compute forward and backward reachable sets of neural networks.
Our algorithm provides a building block for proving safety properties for autonomous systems with learned perception, dynamics, or control components in the loop.  
Specifically, given a set in the input space, RPM computes the set of all corresponding outputs. 
Similarly, given a set of outputs, RPM computes the set of all corresponding inputs under the ReLU network.  
We use these capabilities to compute control invariant sets and regions of attraction (ROAs) for dynamical systems represented as neural networks.
Computing these sets helps roboticists (i) verify whether safety specifications on the robot state are met, (ii) identify states that will converge to some desirable equilibria, and (iii) identify regions in the state space for which a system can be controlled.
Identification of these sets is an important part of the control design process, as directly synthesizing control policies that meet invariance constraints by construction is challenging \cite{blanchini2008set}.
If a constraint is not met, the reachable or invariant sets computed give insight into which states lead to violation of the constraints, informing targeted policy improvements.



\begin{figure}[!t]
  \centering
  \includegraphics[width=0.85\columnwidth]{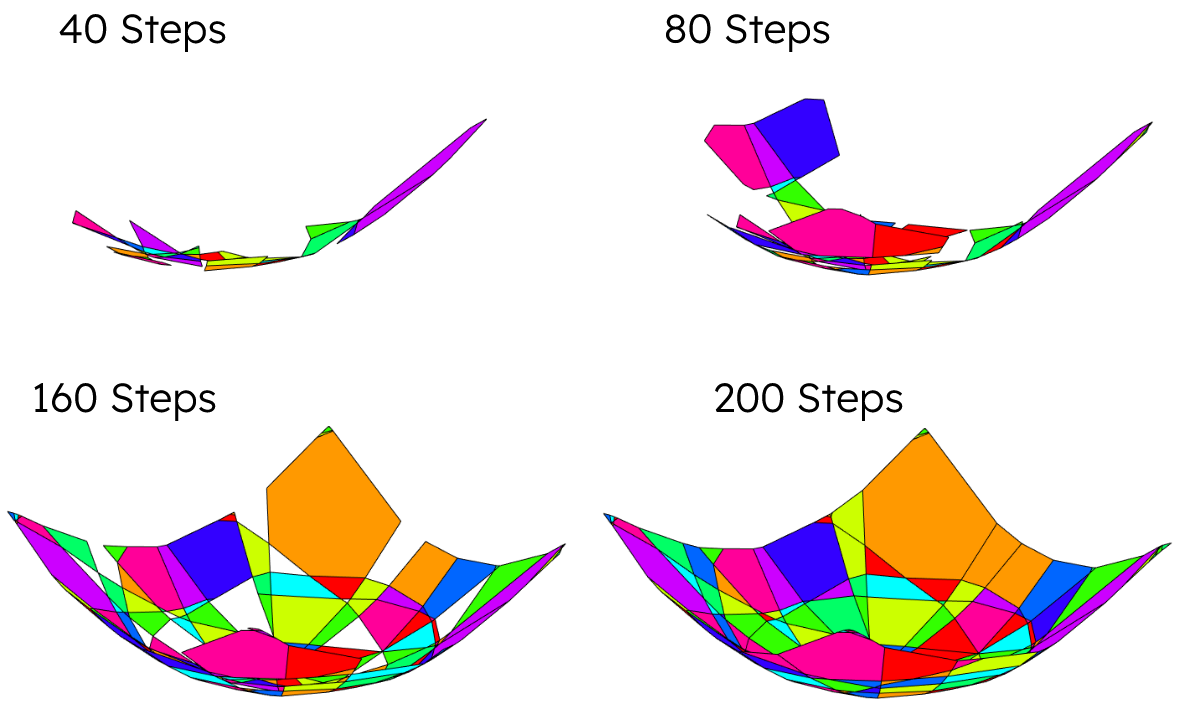}
  \caption{RPM's incremental enumeration of each affine region in a ReLU network. In this case the ReLU network was trained to approximate a quadratic function. Each new affine region enumerated is connected to a previous region.}
  \label{fig:cell_enum}
\end{figure}

It is well known that ReLU networks implement continuous piecewise-affine (PWA) functions, that is, the input space for a ReLU network can be tessellated into polyhedra, and over each polyhedron the neural network is affine. 
The RPM algorithm explicitly finds this equivalent PWA representation for a given ReLU network. 
Figure \ref{fig:cell_enum} illustrates how RPM incrementally solves for this representation. 
The algorithm incrementally enumerates the polyhedra and affine pieces of the PWA function by solving a series of Linear Programs (LPs), and following analytical edge flipping rules to determine neighboring polyhedra. 
The algorithm starts with an initial polyhedron, then solves for its neighboring polyhedra, then neighbors of neighbors, etc., until the desired input set is tessellated. 
In this way, our method is geometrically similar to fast marching methods in optimal control \cite{HJBFastMarching}, path planning \cite{PlanningFastMarching,FMT}, and graphics \cite{FastMarchingCubes,lei2020analytic}. 

We then use methods for PWA reachability to compute forward and backward reachable sets for each polyhedron and associated affine function. 
Computing forward and backward reachable sets allows us to search for and construct control invariant sets and ROAs without Lyapunov tools.
Finally, we also propose an accelerated backward reachability procedure that leverages the connected enumeration of the RPM algorithm to restrict the number of polyhedra the algorithm searches over.
We show that this accelerated procedure is guaranteed to enumerate the entire backward reachable set in the case that the deep network is a homeomorphism (a continuous bijection with continuous inverse).  
Checking this condition is simple, and represents a novel general procedure for examining the invertibility of neural networks whose hidden layers are not invertible themselves.


\subsection{Existing Work}
Most existing algorithms that compute exact reachable sets of neural networks iterate through the network layer-by-layer \cite{xiang2017reachable,nnv,yang2020reachability}. 
The layer-by-layer approaches obtain the entire reachable set at once at the end of the computation, rather than revealing the reachable set piece by piece throughout the computation. 
Consequently, if the computation must end early due to memory constraints or a computer fault, no usable result is obtained. 
In contrast, RPM builds the reachable set one polyhedron at a time, leading to a partial, but still potentially useful result if computation is halted before the algorithm runs to completion. 

\textcolor{black}{The authors of \cite{nnenum} also propose an incremental approach to visiting each affine region in a ReLU network.}
The difference between this tool and our RPM algorithm is that RPM enumerates affine regions of the neural network in a connected walk; each new region that is revealed is connected to a previous region. 
In contrast, the algorithm used in \cite{nnenum} does not have this property; whether a new region is connected to a previous one is unknown by the algorithm.
In Section \ref{Sec:ROA} we show how the connected enumeration property leads to an accelerated algorithm for computing ROAs.
Another way to conceptualize this difference between RPM and \cite{nnenum} is that RPM can also return the graph that describes which affine regions neighbor one another. \textcolor{black}{This feature is enabled by RPM's unique property of removing redundant constraints in the input-space polyhedra, resulting in the minimal description of each affine region.}

 \subsection{Contributions and Organization}
\textcolor{black}{In \cite{VincentSchwagerICRA21RPM}, the RPM algorithm was introduced and demonstrated on multi-step reachability tasks. 
In contrast, the contributions of this paper are
\begin{itemize}
     \item \textcolor{black}{A formal proof of correctness for the RPM algorithm via} a theorem to analytically determine an \textcolor{black}{affine region} from its neighbor by flipping neuron activations in ReLU networks. 
     \item Application of PWA analysis tools to compute control invariant sets and regions of attraction for dynamical systems represented as ReLU networks.
     \item An accelerated exact backward reachability algorithm for ReLU networks \textcolor{black}{that does not require enumerating all affine regions}, leveraging a novel \textcolor{black}{approach} for checking the invertibility of the network.
\end{itemize}}

\textcolor{black}{Through these contributions,} we are able to compute intricate invariant sets that may be non-convex and even disconnected.
In numerical examples, we compute ROAs for a learned van der Pol oscillator, and show this computation enjoys a $15$x speedup because the learned dynamics are homeomorphic. 
We then pair RPM with the MPT3 \cite{MPT3} Matlab toolbox to find an a control invariant set for a learned torque-controlled pendulum. 
Finally, we apply our algorithm to find a set of states stabilized by an image-based airplane runway taxiing system, TaxiNet \cite{verify_gan}. 
The closed-loop system is PWA with ${\sim}250,000$ regions, orders of magnitude larger than PWA systems for which other nonlinear stability analysis approaches have been demonstrated.

The paper is organized as follows.
We give related work in Section \ref{Sec:RelatedWork} and give background and state the problem in Section \ref{Sec:Background}. 
Section \ref{Sec:PWA} describes the RPM algorithm and explains its derivation. 
In Section \ref{Sec:Reachability} we describe how RPM is used to perform forward and backward reachability computations for ReLU networks over multiple time steps. 
In Section \ref{Sec:ROA} we describe how RPM is leveraged to compute control invariant sets and ROAs. 
Finally, Section \ref{Sec:Examples} presents numerical results for computing control invariant sets and ROAs for learned dynamical systems and we offer conclusions in Section \ref{Sec:Conclusions}. 



\section{Related Work} 
\label{Sec:RelatedWork}
Though the analysis of neural networks is a young field, a broad literature has emerged to address varied questions related to interpretability, trustworthiness, and safety verification. Much work has been dedicated to characterizing the expressive potential of ReLU networks by studying how the number of affine regions scales with network depth and width \cite{bengio,Hanin_2019,understanding}. Other research includes encoding piecewise-affine (PWA) functions as ReLU networks \cite{fem,reverseengineering}, learning deep signed distance functions and extracting level sets \cite{lei2020analytic}, and learning neural network dynamics or controllers that satisfy stability conditions \cite{zico,jin2020neural}, which may be more broadly grouped with correct-by-construction training approaches \cite{art,lincolnlab}.

Spurred by pioneering methods such as Reluplex \cite{reluplex}, the field of neural network verification has emerged to address the problem of analyzing properties of neural networks over continuous input sets. A survey of the neural network verification literature is given by \cite{survey}. Reachability approaches are a subset of this literature and are especially useful for analysis of learned dynamical systems. 

Reachability methods can be categorized into overapproximate and exact methods. 
The reachable sets computed by overapproximate methods are guaranteed to capture all reachable states, but may also contain states that are unreachable.
Overapproximate methods often compute neuron-wise bounds either from interval arithmetic or symbolically \cite{maxsens,fastlip,reluval,neurify,crown, xiang2020reachable}. Optimization based approaches are also used to solve for bounds on a reachable output set in \cite{sherlock,alessio1,alessio2,proveable}. Other approaches include modeling the network as a hybrid system \cite{verisig}, abstracting the domain \cite{abstract}, and performing layer-by-layer operations on zonotopes \cite{AI2}.
Closed-loop forward reachability using overapproximate methods is investigated in \cite{julian2019reachability, everett, overt, entesari2023automated, meng2022learning, kochdumper2023open}
Closed-loop backward reachability using overapproximate methods is investigated in \cite{rober2023backward, zhang2023backward}.

Exact reachability methods have also been proposed, although to a lesser degree. 
These methods either iteratively refine the reachable set by applying layer-by-layer transformations \cite{xiang2017reachable,nnv,yang2020reachability}, or they solve for the explicit PWA representation and compute reachable sets from this \cite{nnenum, VincentSchwagerICRA21RPM}.
Similar layer-by-layer approaches have also been proposed to solve for the explicit PWA representation of a ReLU network \cite{robinson2020dissecting, hanin2019complexity, xu2022traversing}.

Our RPM algorithm is an exact method, meaning that there is no conservatism in the reachable sets computed by our algorithm.
Our algorithm inherits all the advantages of exact methods, but is unique in that it enumerates the reachable set in an incremental and connected walk.
This makes our high-level procedure similar to explicit model predictive control methods \cite{eMPC_book}. 
Finally, all intermediate polyhedra and affine map matrices of layer-by-layer methods must be stored in memory until the algorithm terminates, whereas with incremental methods, once a polyhedron-affine map pair is computed it can be sent to external memory and only a binary vector (the neuron activations) needs to be stored to continue the algorithm. 
This is especially useful because no method has been shown to accurately estimate the number of regions without explicit enumeration.
Unlike other incremental approaches, ours explores affine regions in a connected walk. As a point of distinction, this allows our algorithm to not only output the PWA representation, but for each affine region also output the neighboring affine regions.

In this paper we demonstrate how RPM can be used not only for finite-time reachability, but also for computing control invariant sets and ROAs.
Computing these sets tends to be much harder than computing finite-time reachable sets, but the associated guarantees hold for all time. 
Few methods exist for computing invariant sets or ROAs for dynamical systems represented as neural networks, and most that do rely on an optimization-based search of a Lyapunov function. 
In \cite{ellip_roa}, the authors learn a quadratic Lyapunov function using semidefinite programming that results in an ellipsoidal ROA, a drawback of which is that ellipsoids are not very expressive sets. 
In contrast, nonconvex ROAs can be computed by \cite{richards2018lyapunov, chen2021learning, dai2021lyapunov}. 
These methods learn Lyapunov functions for a given dynamical system and verify the Lyapunov conditions either using samples and Lipschitz continuity or mixed integer programming. 
In \cite{control_inv_changliu}, an optimization and sampling-based approach for synthesizing control invariant sets is given, but as with searching for Lyapunov functions, it may fail in finding a control invariant set when one exists. 
All of these methods inherit the drawbacks of Lyapunov-style approaches wherein a valid Lyapunov function may not be found when one exists, and certifying the Lyapunov conditions is challenging.
In contrast to these methods, we explicitly solve for the PWA form of the dynamics, then find invariant sets using reachability methods. 
A reachability-based approach like ours can be more constructive and reliable than Lyapunov approaches, as we show in our experiments.
Lastly, a non-Lyapunov method is given in \cite{jouret2023safety}, where the focus is on continuous-time dynamical systems, whereas our focus is on discrete-time.


For PWA dynamics, there are specialized algorithms for computing ROAs based on convex optimization approaches for computing Lyapunov functions \cite{biswas, pwa_lyap, lmi_lyap, baldi2018reachable}, or reachability approaches \cite{MPT3,pwa_lygeros}. 
A drawback to the Lyapunov approaches for PWA dynamics is they can be too conservative and the standard approaches require the user to provide an invariant domain, which itself can be very challenging to find. 
Furthermore, in Section 7.2 of \cite{biswas}, motivating examples are given where a variety of Lyapunov approaches fail to find valid Lyapunov functions for very small PWA systems that are known to be stable. 

In contrast to Lyapunov approaches, the reachability approach requires an initial `seed' set of states that is a ROA and uses backward reachability to grow the size of the ROA.
Finding a seed ROA can be difficult for general nonlinear systems, but for PWA systems can be easily found due to the dynamics being locally linear.
Although reachability-based computation of invariant sets has been explored for PWA dynamical systems before \cite{pwa_lygeros}, a novelty in our work is pairing these methods with our RPM algorithm so that they can apply to dynamical systems represented as neural networks.
Our proposed method for finding ROAs of neural networks follows the backward reachability approach from the literature. 
In addition, we show that backward reachability may be sped up considerably when restricting RPM to enumerate a connected backward reachable set (a unique feature enabled by the connected walk of the algorithm).

\section{Background and Problem Statement}
\label{Sec:Background}
We begin by formalizing a model of a ReLU network, and defining various concepts related to polyhedra, PWA functions, and dynamical systems.
We then state the main problems we seek to address in this paper.
\subsection{ReLU Networks} \label{sec:relu_networks}
An $L$-layer feedforward neural network implements a function $\by = F_{\bTheta}(\bx)$, with the map $F_{\bTheta}: \R^{n} \rightarrow \R^{m}$ defined recursively by the iteration
\begin{align}
\label{Eq:NNBasic}
    \bz_{i} = \sigma_i(\bW_i \bz_{i-1} + \bb_i), \quad i = 1, \ldots, L
\end{align}
where $\bz_0 = \bx$ is the input, $\by = \bz_L$ is the output, and $\bz_i$ is the hidden layer output of the network at layer $i$.  The function $\sigma_i(\cdot)$ is the activation function at layer $i$, and $\bW_i\in\R^{l_i \times l_{i-1}}$ and $\bb_i\in\R^{l_i}$ are the weights and biases, respectively, where $l_i$ denotes the number of neurons in layer $i$. We assume $\sigma_{L}(\cdot)$ is an identity map and all hidden layer activations are ReLU,
\begin{align}
\label{Eq:ReLU}
    \sigma_i(\bz^\text{pre}_{i}) = [\max(0,z^\text{pre}_{i1}) \, \cdots \, \max(0,z^\text{pre}_{i l_i})]^\top,
\end{align}
where $\bz^\text{pre}_i = [z^\text{pre}_{i1} \, \cdots \, z^\text{pre}_{i l_i}]^\top = \bW_i\bz_{i-1} + \bb_i$ is the pre-activation hidden state at layer $i$. We can rewrite the iteration in (\ref{Eq:NNBasic}) in a homogeneous form
\begin{align}
\label{Eq:NNHomogeneous}
        \begin{bmatrix}
            \bz_i \\ 
            1
        \end{bmatrix}
            &= \sigma_i\Big(\bTheta_i \begin{bmatrix}
            \bz_{i-1} \\ 
            1
        \end{bmatrix}\Big),
\end{align}
where
\begin{subequations}
\begin{align}
\bTheta_i &= 
        \begin{bmatrix}
            \bW_i & \bb_i \\
            \mathbf{0}^\top & 1
        \end{bmatrix} \quad \text{for} \quad i = 1, \ldots, L-1, \label{Eq:ThetaDef} \\
\bTheta_L &= 
    \begin{bmatrix}
        \bW_L & \bb_L
    \end{bmatrix}. \label{Eq:ThetaLDef}
\end{align}
\end{subequations}
We define $\bTheta = (\bTheta_1, \ldots, \bTheta_L)$ as the tuple containing all the parameters that define the function $F_{\bTheta}(\cdot)$. 

Given an input $\bx$ to a ReLU network, every hidden neuron has an associated binary \textit{neuron activation} of zero or one, corresponding to whether the preactivation value is nonpositive or positive, respectively.\footnote{This choice is arbitrary. Some works use the opposite convention.} Specifically, the activation for neuron $j$ in layer $i$ is given by
\begin{align}
\label{Eq:AP}
    \lambda_{(ij)}(\bx) &= \begin{cases}
            1 \quad \text{for} \quad z^\text{pre}_{ij}(\bx) > 0 \\    
            0 \quad \text{for} \quad z^\text{pre}_{ij}(\bx) \le 0. 
                        \end{cases}
\end{align}
We define the activation pattern at layer $i$ as the binary vector $\blambda_{(i)}(\bx) = [\lambda_{(i1)}(\bx) \, \cdots \, \lambda_{(i l_i)}(\bx)]^\top$ and the activation pattern of the whole network as the tuple $\lambda(\bx) = (\blambda_{(1)}(\bx), \ldots, \blambda_{(L-1)}(\bx))$.  For a network with $N = \sum_{i = 1}^{L-1}l_i$ neurons there are $2^N$ possible combinations of neuron activations, however not all are realizable by the network due to inter-dependencies between neurons from layer to layer, as we will see later in the paper. Empirically, the number of activation patterns achievable by the network is better approximated by $N^{n}$ for input space with dimension $n$  \cite{hanin2019deep}. 

We can write equivalent expressions for the neural network function in terms of the activation pattern.  Define the diagonal activation matrix for layer $i$ as $\bLambda_{(i)}(\bx) = \text{diag}([\blambda_{(i)}(\bx)^\top \, 1]^\top)$, where a $1$ is appended at the end\footnote{Note that the homogeneous form adds a dummy neuron at each layer that is always active, since the last element of the hidden state is $1$ in homogeneous form, and thus always positive.} to match the dimensions of the homogeneous form in (\ref{Eq:NNHomogeneous}). Using the activation matrix, we can write the iteration defining the ReLU network from (\ref{Eq:NNHomogeneous}) as
\begin{align}
    \label{Eq:NNActivationDef}
    \begin{bmatrix}
    \bz_i\\
    1
    \end{bmatrix}
     = \bLambda_{(i)}(\bx)\bTheta_i\begin{bmatrix} \bz_{i-1}\\
    1
    \end{bmatrix}.
\end{align}
Finally, the map from input $\bx$ to the hidden layer preactivation $z_{ij}^\text{pre}$ can be written explicitly from the iteration (\ref{Eq:NNActivationDef}) as 
\begin{align}
\label{Eq:PreActExplicit}
    z_{ij}^\text{pre}(\bx) = \boldsymbol{\theta}_{ij}\Big(\prod_{l = 1}^{i-1}\bLambda_{(l)}(\bx)\bTheta_l\Big)\begin{bmatrix} \bx\\
    1
    \end{bmatrix},
\end{align}
where $\boldsymbol{\theta}_{ij}$ is the $j^{\text{th}}$ row of the matrix $\bTheta_i$. Similarly, the whole neural network function $\by = F_{\bTheta}(\bx)$ can be written explicitly as a map from $\bx$ to $\by$ as
\begin{align}
\label{Eq:NNExplicit}
    \by = \bTheta_L\Big(\prod_{i = 1}^{L-1}\bLambda_{(i)}(\bx)\bTheta_i\Big)\begin{bmatrix} \bx\\
    1
    \end{bmatrix}.
\end{align}
It is well known that in  (\ref{Eq:NNExplicit}), the output $\by$ is a continuous and PWA function of the input $\bx$ \cite{bengio,Hanin_2019,understanding}.  Next, we mathematically characterize PWA functions and give useful definitions.

\subsection{Polyhedra and PWA Functions}
\begin{definition}[(Convex) Polyhedron]
A convex polyhedron is a closed convex set $P \subset \mathbb{R}^n$ defined by a finite collection of $M$ halfspace constraints,
\[
P = \{\bx \mid \ba_i^\top \bx \le b_i, \quad i = 1, \ldots, M\},
\]
where $\ba_i,\bx \in \R^n$, $b_i \in \R$, and each $(\mathbf{a}_i, b_i)$ pair defines a halfspace constraint. Defining matrix $\mathbf{A} = [\mathbf{a}_1 \, \cdots \, \mathbf{a}_N]^\top$ and vector $\mathbf{b} = [b_1 \, \cdots \, b_N]^\top$, we can equivalently write
\begin{align*}
    P = \{\bx \mid \mathbf{Ax} \le \mathbf{b} \}. 
\end{align*}
We henceforth refer to convex polyhedra as just polyhedra. 
\end{definition}

The above representation of a polyhedron is known as the halfspace representation, or H-representation. 
Note:
\begin{itemize}
\item A polyhedron can be bounded or unbounded.
\item A polyhedron can occupy a dimension less than the ambient dimension (if $(\mathbf{a}_i, b_i) = \alpha_{ij}(-\mathbf{a}_j, -b_j)$ for some pairs $(i,j)$ and some positive scalars $\alpha_{ij}$).
\item A polyhedron can be empty ($\nexists \bx$ such that $\mathbf{Ax} \le \mathbf{b}$).
\item Without loss of generality, the halfspace constraints for a polyhedron can be normalized so that $\|\ba_i\|\in \{0,1\}$ (by dividing the unnormalized parameters $\ba_i$ and $b_i$ by $\|\ba_i\|$), where $\|\cdot\|$ is the $\ell_2$ norm. $\|\ba\| = 0$ represents a degenerate constraint that is trivially satisfied if $b_i \ge 0$.
\item An H-representation of a polyhedron may have an arbitrary number of redundant constraints, which are constraints that are implied by other constraints. 
\end{itemize}

\begin{definition}[Polyhedral Tessellation]\label{def:tessellation}
A polyhedral tessellation $\mathcal{P} = \{P_1, \ldots, P_N\}$ is a finite set of polyhedra that tessellate a set $\mathcal{X} \subset \mathbb{R}^n$. That is, $\cup_{i = 1}^NP_i = \mathcal{X}$ and $|P_i\cap P_j|_n = 0$, for all $i \not= j$, where $|\cdot|_n$ denotes the $n$-dimensional Euclidean volume of a set (the integral over the set with respect to the Lebesgue measure of dimension $n$). 
\end{definition}

Intuitively, the polyhedra in a tessellation together cover the set $\mathcal{X}$, and can only intersect with each other at shared faces, edges, vertices, etc. 
In the remainder of the paper when we refer to shared faces between neighboring polyhedra we mean the $n-1$ dimensional polyhedron given by the set intersection $P_i \cap P_j$.

\begin{definition}[Neighboring Polyhedra]
Two polyhedra, $P_i$ and $P_j$, are neighbors if their intersection has non-zero Euclidean volume in $n-1$ dimensions, that is, $|P_i \cap P_j|_{n-1} > 0$. 
\end{definition}

Intuitively, neighboring polyhedra are adjacent to one another in the tessellation, and share a common face. A polyhedral tessellation naturally induces a graph in which each node is a polyhedron and edges exist between neighboring polyhedra.

\begin{definition} [Piecewise Affine (PWA) Function] A PWA function is a function $F_{\text{PWA}}(\bx) : \mathbb{R}^{n} \rightarrow \mathbb{R}^{m}$ that is affine over each polyhdron in a polyhedral tessellation $\mathcal{P}$ of $\mathbb{R}^n$.  Specifically, $F_{\text{PWA}}(\bx)$ is defined as
\begin{align}
    F_{\text{PWA}}(\bx) = \mathbf{C}_k \bx + \mathbf{d}_k \ \ \ \forall \bx \in P_k \ \ \ \forall P_k \in \mathcal{P}, \label{eq:PWA}
\end{align}
where $\mathbf{C}_k \in \mathbb{R}^{m \times n}$, $\mathbf{d}_k \in \mathbb{R}^{m}$.
\end{definition}

Note that this requires a collection of tuples $\{(\mathbf{A}_1,\mathbf{b}_1, \mathbf{C}_1, \mathbf{d}_1), \ldots, (\mathbf{A}_N,\mathbf{b}_N, \mathbf{C}_N, \mathbf{d}_N)\}$, where $\mathbf{A}_k, \mathbf{b}_k$ define the polyhedron, and $\mathbf{C}_k, \mathbf{d}_k$ define the affine function over that polyhedron. 
We refer to the PWA representation in (\ref{eq:PWA}) as the \textit{explicit} PWA representation, as each affine map and polyhedron is written explicitly without specifying the relationship between them. 
There also exist other representations, such as the lattice representation \cite{lattice_orig,lattice_irred} and the, so called, canonical representation \cite{canonical, pwa_comp}. 

\subsection{Dynamical Systems}
In this paper we are concerned with analyzing dynamical systems represented by ReLU networks.
Specifically, we consider the situations when the ReLU network represents the state transition function for a controlled (\ref{eq:controlled}) or autonomous (\ref{eq:uncontrolled}) dynamical system,
\begin{align}
    \bx_{t+1} &= F_{\bTheta}(\bx_t, \mathbf{u}_t) \label{eq:controlled} \\
    \bx_{t+1} &= F_{\bTheta}(\bx_t) \label{eq:uncontrolled}.
\end{align}
For notational convenience, in the remainder of the paper we drop the $\bTheta$ subscript from $F$.
We next define key concepts and state the overall problems that we seek to solve.

\begin{definition}[$t$-Step Forward Reachable Set]
    Given an initial set of states $\mathcal{X}$ and a set of admissible control inputs $\mathcal{U}$, a $t$-step forward reachable set can be defined recursively,
    \begin{gather*}
    \mathcal{S}_0 = \mathcal{X} \\
        \mathcal{S}_{t} = \{\bx_{t} \mid \bx_{t} = F(\bx_{t-1}, \mathbf{u}_{t-1}),\\ \forall \bx_{t-1} \in \mathcal{S}_{t-1},\ \forall \mathbf{u}_{t-1} \in \mathcal{U}\}.
    \end{gather*}
\end{definition}

\begin{definition}[$t$-Step Backward Reachable Set]
    Given a final set of states $\mathcal{X}$ and a set of admissible control inputs $\mathcal{U}$, a $t$-step backward reachable set can be defined recursively,
    \begin{gather*}
    \mathcal{S}_0 = \mathcal{X} \\
        \mathcal{S}_{-t} = \{\bx_{-t} \mid \bx_{-t+1} = F(\bx_{-t}, \mathbf{u}_{-t}),\\ \forall \bx_{-t+1} \in \mathcal{S}_{-t+1},\ \forall \mathbf{u}_{-t} \in \mathcal{U}\}.
    \end{gather*}
\end{definition}

The ability to compute $t$-step forward and backward reachable sets will later enable us to compute control invariant sets and regions of attraction.
\begin{definition}[Control Invariant Set]
    Given a set of admissible control inputs $\mathcal{U}$, a control invariant set is a set of states for which there exists a sequence of control inputs such that the system state remains in the set for all time.
    If $\mathcal{S}$ is a control invariant set, then 
    \begin{gather*}
        \bx_0 \in \mathcal{S} \implies \exists \mathbf{u}_0, \ldots, \mathbf{u}_{t-1} \in \mathcal{U} \text{  s.t.  } \bx_t \in \mathcal{S} \quad \forall t \in \mathbb{N}
    \end{gather*}
    where $\bx_t = F(\bx_{t-1}, \mathbf{u}_{t-1})$ and $\mathbb{N}$ is the natural numbers.
\end{definition}

\begin{definition}[(Invariant) Region of Attraction]
    Given an autonomous dynamical system $\bx_{t+1} = F(\bx_t)$, an invariant region of attraction (ROA) is a set of states that asymptotically converges to an equilibrium state ($\bx_{\text{eq}}$) and for which sequences of states remain for all time. 
    If $\mathcal{S}$ is an ROA, then
    \begin{align*}
        \bx_0 \in \mathcal{S} &\implies \bx_t \in \mathcal{S} \quad \forall t \in \mathbb{N} \\
        \bx_0 \in \mathcal{S} &\implies \lim_{t \rightarrow \infty} \bx_t = \bx_{\text{eq}}
    \end{align*}
    where $\bx_t = F(\bx_{t-1})$.
\end{definition}


In this paper we seek to solve the following problems.
\begin{problem}
Given a ReLU network $\by = F(\bx)$ and a polyhedral input set $\mathcal{X} \subset \mathbb{R}^n$, compute the forward reachable set $\mathcal{Y}$. Similarly, given a polyhedral output set $\mathcal{Y}\subset \mathbb{R}^m$ compute the backward reachable set $\mathcal{X}$.
\end{problem}
\begin{problem}
Given a ReLU network that implements a discrete-time dynamical system $\bx_{t+1} = F(\bx_t)$ and state domain $\mathcal{X}$, identify the existence of stable fixed-point equilibria $\bx_{eq}$ and compute associated ROAs.
\end{problem}
\begin{problem}
Given a ReLU network that implements a discrete-time dynamical system $\bx_{t+1} = F(\bx_t, \mathbf{u}_t)$, a state domain $\mathcal{X}$, and set of admissible inputs $\mathcal{U}$, compute a control invariant set (if one exists).
\end{problem}
\begin{problem}
Given a ReLU network that implements a function $\by = F(\bx)$, determine whether the network is a homeomorphism (bijection with continuous inverse).
\end{problem}
In the next section we describe the RPM algorithm for transcribing the ReLU network function $\by = F_{\bTheta}(\bx)$ into its equivalent explicit PWA representation $\by = F_{\text{PWA}}(\bx)$. We then address the solution of the above problems in the subsequent sections.



\section{From ReLU Network to PWA Function}
\label{Sec:PWA}
First, we seek to construct the explicit PWA representation of a ReLU network. Our method enumerates each polyhedral region and its associated affine map directly from the ReLU network. In Sec.~\ref{Cells} we show how polyhedra and affine maps are computed from the activation pattern of a ReLU network. In Sec.~\ref{Essential} we show how polyhedral representations are reduced to a minimal form, which is used in Sec.~\ref{Neighbors} to determine neighboring polyhedra given a current polyhedron. This leads to a recursive procedure, that we call Reachable Polyhedral Marching (RPM), in which we explore an expanding front of polyhedra, ultimately giving the explicit PWA representation, as explained in Sec.~\ref{Cell Enumeration}.

\subsection{Determining Polyhedral Regions from Activation Patterns} 
\label{Cells}

We show here that the network activation pattern $\blambda(\bx)$ from (\ref{Eq:AP}) has a one-to-one correspondence with the regions of the polyhedral tessellation underlying the ReLU network. We show how to explicitly extract the half-space constraints defining the polyhedron from the activation pattern. 

Consider the expression for the preactivation value $z_{ij}^{\text{pre}}$ in (\ref{Eq:PreActExplicit}). We can re-write this equation as
\begin{align}
\label{Eq:PreActHyperplane}
    z_{ij}^{\text{pre}} = [\bar\ba_{ij}^\top \quad  -\bar b_{ij}]\begin{bmatrix}
        \bx\\
        1
    \end{bmatrix},
\end{align}
where
\begin{align}
\label{Eq:HalfspaceDef}
    [\bar\ba_{ij}^\top \quad -\bar b_{ij}] = \boldsymbol{\theta}_{ij}\Big(\prod_{l = 1}^{i-1}\bLambda_{(l)}(\bx)\bTheta_l\Big).
\end{align} 
The activation $\lambda_{(ij)}$ is decided by the test $z_{ij}^\text{pre} > 0$, which we can write as $\bar\ba_{ij}^\top\bx > \bar b_{ij}$, defining a halfspace constraint in the input space.  Specifically, we have the following cases,
\begin{align}
\label{Eq:HalfspaceCases}
    \begin{cases}
          \bar\ba_{ij}^{\top} \bx > \bar b_{ij} \quad \text{if} \quad  \lambda_{(ij)} = 1 \\    
          \bar\ba_{ij}^{\top} \bx \le \bar b_{ij} \quad \text{if} \quad  \lambda_{(ij)} = 0.
      \end{cases}
\end{align}
This defines a halfspace constraint, that may or may not be open (due to the strict `$>$' inequality). However, on the boundary $\bar\ba_{ij}^\top\bx = \bar b_{ij}$, we have $z_{ij}^{\text{pre}} = 0$ from (\ref{Eq:PreActHyperplane}), leading to the post activation hidden state $z_{ij} = \lambda_{(ij)}z_{ij}^{\text{pre}} = 0$. We see that the value of the activation $\lambda_{(ij)}$ is irrelevant on the boundary, as the post activation state evaluates to zero regardless. We can therefore replace the `$>$' with `$\ge$' in (\ref{Eq:HalfspaceCases}) without loss of generality. 

Finally, to obtain the standard normalized form for a halfspace constraint ($\ba_{ij}^\top \bx \le b_{ij}$), we define
\begin{subequations}
\begin{align}
    \ba_{ij} &= (1-2\lambda_{(ij)})\frac{\bar\ba_{ij}}{\|\bar\ba_{ij}\|} \\
    b_{ij} &= (1-2\lambda_{(ij)})\frac{\bar b_{ij}}{\|\bar\ba_{ij}\|},
\end{align}
\label{eq:normalize}
\end{subequations}
where, in the degenerate case when $\bar\ba_{ij} = \mathbf{0}$, we define $\ba_{ij} = \mathbf{0}$ and $b_{ij} = \bar b_{ij}$. Hence, we obtain one halfspace constraint $\ba_{ij}^\top \bx \le b_{ij}$ for each neuron activation state $\lambda_{(ij)}$.

Given a specific input $\bx$, we can then take the resulting activation pattern of the network $\lambda(\bx)$, and directly extract the halfspace constraints that apply at that activation state from (\ref{Eq:HalfspaceCases}). In fact, (\ref{Eq:HalfspaceDef}) shows that $(\ba_{ij}, b_{ij})$ are actually functions of the activation patterns at earlier layers in the network. Indeed, consider perturbing $\bx$ as $\bx' = \bx + \boldsymbol{\delta}$. The halfspace constraints $(\ba_{ij}, b_{ij})$ will remain fixed under this perturbation until $\bdelta$ is large enough to change the activation pattern of an earlier layer $\lambda_{(l)}$, $l<i$. Consider the set of all such perturbed input values $\bx$ that do not result in a change in any neuron activation. We have
\begin{align}
    \label{eq:hrep_2}
    \scalemath{0.94}{P_{\lambda} = \{\bx \mid \ba_{ij}^\top \bx \le b_{ij} \quad i = 1, \ldots, L-1, \, j = 1, \ldots, l_i\}}, 
\end{align}
which is a polyhedron. We see that for each activation pattern $\lambda$, there exists an associated polyhedron $P$ in the input space over which that activation pattern remains constant. The procedure for determining the polyhedron associated with an activation pattern is formalized in Algorithm \ref{Algo:ap_to_poly}. A unique activation pattern $\lambda_k$ can then be defined for each affine region $k$. To be clear, we refer to the activation pattern for region $k$ as $\lambda_k$, and use subscripts in parentheses to refer to the specific layer and neuron activation values. Lastly, for a fixed activation pattern, from (\ref{Eq:NNExplicit}) the ReLU network simplifies to the affine map $\by = \mathbf{C}_k \bx + \mathbf{d}_k$ where
\begin{align}
\label{eq:CdDef}
    \begin{bmatrix}
        \bC_k & \bd_k
    \end{bmatrix} &= 
    \bTheta_L\Big(\prod_{i = 1}^{L-1}\bLambda_{k,(i)}(\bx)\bTheta_i\Big).
\end{align}

\begin{algorithm}[h!]
\setlength\belowcaptionskip{-0.5\baselineskip}
\SetAlgoLined
\DontPrintSemicolon
 \KwIn{$\lambda$}
 \KwOut{$\bA, \bb$}
 $k \leftarrow 1$ \\
\For{$i = 1,\ldots,L-1$, $j = 1,\ldots,l_i$}{
$\Bar{\mathbf{a}}_{ij},\ \Bar{b}_{ij} \leftarrow$ (\ref{Eq:HalfspaceDef}) \\
$\mathbf{a}_{ij},\ b_{ij} \leftarrow$ (\ref{eq:normalize}) \\
$\bA[k,:],\ \bb[k] \leftarrow \mathbf{a}_{ij},\ b_{ij}$ \\
$k \leftarrow k+1$
}
\Return{$\bA, \bb$}
\caption{Polyhedron from Activation Pattern}
\label{Algo:ap_to_poly}
\end{algorithm}

\subsection{Finding Essential Constraints} \label{Essential}
As mentioned previously, a polyhedron may have redundant constraints. We find that many of the constraints for the polyhedron $P_k$ generated by $\lambda_k$ are either duplicates or redundant, and can thus be removed. We define more formally the concepts of duplicate and redundant constraints.
\begin{definition}[Duplicate Constraints]
A constraint $\mathbf{a}_j^{\top} \bx \le b_j$ is duplicate if there exists a scalar $\alpha > 0$ and a constraint $\mathbf{a}_i^{\top} \bx \le b_i$, $i<j$, such that $\alpha [\mathbf{a}^\top_j\ b_j] = [\mathbf{a}^\top_i\ b_i]$. \label{def: duplicate}
\end{definition}
\begin{definition}[Redundant \& Essential Constraints]
A constraint is redundant if the feasible set does not change upon its removal. An essential constraint is not redundant. \label{def: redundant}
\end{definition}

We next describe how to remove the redundant constraints in a H-representation, leaving only the essential halfspace constraints. We first normalize all constraints, remove any duplicate constraints, and consider the resulting  H-representation $\mathbf{A}\bx \le \mathbf{b}$. To determine if the remaining $i^{\text{th}}$ constraint is essential or redundant, we define a new set of constraints with the $i^{\text{th}}$ constraint removed,
\begin{subequations}
    \label{eq:pre-LP}
    \begin{align}
        \tilde{\mathbf{A}} = [\mathbf{a}_1 \ldots \mathbf{a}_{i-1}\ \mathbf{a}_{i+1} \ldots \mathbf{a}_M]^\top \\
         \tilde{\mathbf{b}} = [b_1 \ldots b_{i-1}\ b_{i+1} \ldots b_M]^\top,
    \end{align}
\end{subequations}
and solve the linear program
\begin{subequations}
    \label{eq:LP}
    \begin{align}
        \max_{\bx} \quad & \mathbf{a}^{\top}_i \bx\\
        \textrm{subject to} \quad & \tilde{\mathbf{A}}\bx \le \tilde{\mathbf{b}}.
    \end{align}
\end{subequations}
If the optimal objective value is less than or equal to $b_i$, constraint $i$ is redundant. 
Note, it is critical that any duplicate constraints are removed before this procedure. We formalize this procedure in Algorithm \ref{Algo:essential hrep}.
\begin{algorithm}[h!]
\SetAlgoLined
\DontPrintSemicolon
 \KwIn{$\lambda$}
 \KwOut{$\bA, \bb$}
 $\bA, \bb \leftarrow$ Algorithm \ref{Algo:ap_to_poly} \\
 $\bA, \bb \leftarrow$ remove duplicate constraints of $\bA, \bb$  \\
 $\tilde{\bA}, \tilde{\bb} \leftarrow$ $\bA, \bb$ \algorithmiccomment{make copy of $\bA, \bb$} \\
\For{$\ba_i, b_i \in \bA, \bb$}{
$\tilde{\bA}, \tilde{\bb} \leftarrow$ remove $\ba_i, b_i$ from $\tilde{\bA}, \tilde{\bb}$ \\
$p^* \leftarrow$ optimal objective value of (\ref{eq:LP})  \\
\If{$p^* > b_i$}
{add $\ba_i, b_i$ back to $\tilde{\bA}, \tilde{\bb}$}

}
$\bA, \bb \leftarrow \tilde{\bA}, \tilde{\bb}$ \\
\Return{$\bA, \bb$}
\caption{Essential H-representation}
\label{Algo:essential hrep}
\end{algorithm}

In the worst case, a single LP must be solved for each constraint to determine whether it is essential or redundant. However, heuristics exist to practically avoid this worst case complexity. For computational efficiency, Algorithm \ref{Algo:essential hrep} can be modified to first identify and remove a subset of redundant constraints using the bounding box heuristic \cite{morari}. We observe that this can result in identifying as many as $90\%$ of the redundant constraints. We find that other heuristics such as ray-shooting do not improve performance in our tests.


\begin{figure*}[t]
    \setlength\belowcaptionskip{-0.5\baselineskip}
     \centering
     \begin{subfigure}[b]{0.24\textwidth}
         \centering
         \includegraphics[width=\textwidth]{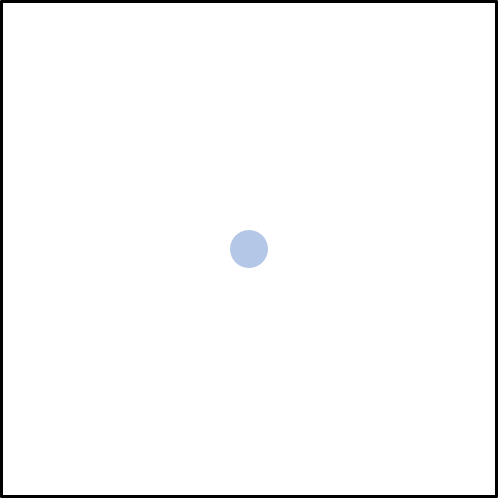}
         \caption{}
     \end{subfigure}
     \hfill
     \begin{subfigure}[b]{0.24\textwidth}
         \centering
         \includegraphics[width=\textwidth]{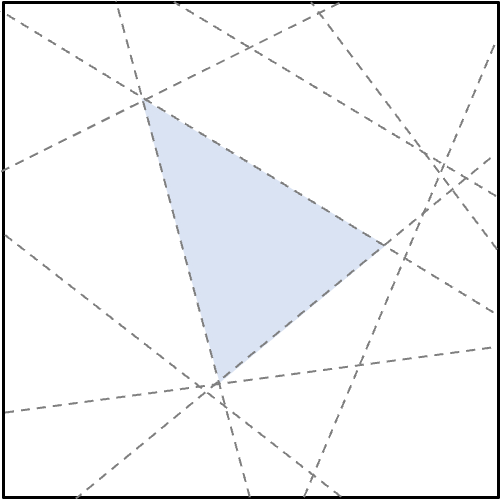}
         \caption{}
     \end{subfigure}
     \hfill
     \begin{subfigure}[b]{0.24\textwidth}
         \centering
         \includegraphics[width=\textwidth]{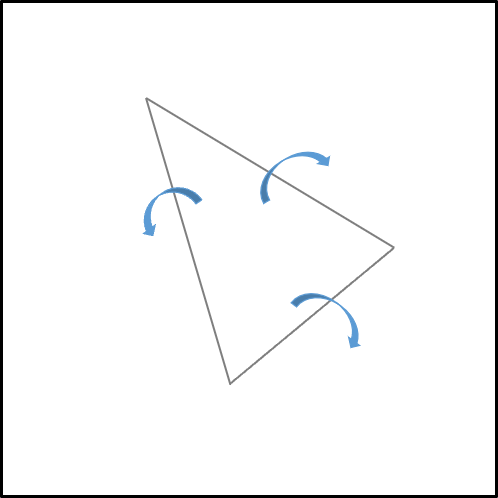}
         \caption{}
     \end{subfigure}
     \hfill
     \begin{subfigure}[b]{0.24\textwidth}
         \centering
         \includegraphics[width=\textwidth]{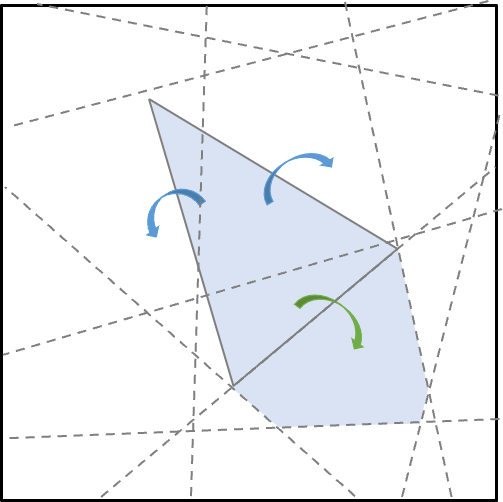}
         \caption{}
     \end{subfigure}
        \caption{This figure illustrates the main steps of the RPM algorithm (Algorithm \ref{Algo:cell_enumeration}). (a) The algorithm is initialized with a point in the input space that is evaluated by the ReLU network to determine the activation pattern $\lambda$. (b) Then the polyhedron corresponding to $\lambda$ is computed using Algorithm \ref{Algo:ap_to_poly}. (c) Next, the essential constraints for the polyhedron, which in turn connect it to unexplored neighboring regions (indicated by blue arrows), are determined using Algorithm \ref{Algo:essential hrep}. (d) Then the activation pattern for a neighboring region (indicated by the green arrow) is found using Algorithm \ref{Algo:neighbor_ap}. The process then repeats by recursively exploring neighboring regions in this manner until the entire input set is explored.}
        \label{fig:alg_steps}
\end{figure*}

\subsection{Determining Neighboring Activation Patterns} \label{Neighbors}
Consider two neighboring polyhedra $P_k$ and $P_{k'}$ such that their shared face is identified by $\ba_\eta^\top \bx = b_\eta$ and $\ba_\eta^\top \bx \le b_\eta$ is an essential constraint for $P_k$, while $-\ba_\eta^\top \bx \le -b_\eta$ is an essential constraint for $P_{k'}$. We call $\ba_\eta^\top \bx \le b_\eta$ the neighbor constraint. Given the neuron activation pattern $\lambda$ for $P_k$, we next describe a procedure to find the activation pattern $\lambda'$ for $P_{k'}$, from which we can find the essential halfspace constraints from the procedure described in the previous section. This allows us to systematically generate all the activation patterns that are realizable by the network, and to obtain the polyhedra and affine maps associated with each of those activation patterns.  

Intuitively, to generate $\lambda_{k'}$ one ought to simply flip the activation of the neurons defining $\ba_\eta^\top \bx \le b_\eta$, and compute all the resulting new halfspace constraints from later layers in the network from (\ref{Eq:HalfspaceDef}). However, this intuitive procedure is incomplete because it does not correctly account for the influence that one neuron may have on other neurons in later layers. Flipping one neuron may lead to other neurons in later layers being flipped as well. To correctly determine the downstream effects of flipping a neuron, we provide a theorem for when neuron activations must be flipped.

\begin{theorem}[Neighboring Activation Patterns]
\label{thm: activation_flipping}
For a given region $P_k$, each neuron defines a hyperplane $\ba_{k,ij}^\top \bx = b_{k,ij}$, as given by (\ref{eq:normalize}). Suppose regions $P_k$ and $P_{k'}$ neighbor each other and are separated by $\ba_\eta^\top \bx = b_\eta$ where $P_{k} \subset \{\bx \mid \ba_\eta^\top \bx \le b_\eta \}$. Then,
\begin{enumerate}
    \item if $[\ba_{k,ij}^\top\quad  -b_{k,ij}] \neq \alpha[\ba_\eta^\top\quad -b_\eta]$ for some $\alpha \in \{0,1\}$, the activation of neuron $ij$ does not change, $\lambda_{k,ij} = \lambda_{k',ij}$.
    \item If $[\ba_{k,ij}^\top\quad  -b_{k,ij}] = \alpha[\ba_\eta^\top\quad -b_\eta]$ for some $\alpha \in \{0,1\}$, then $[\ba_{k',ij}^\top\quad  -b_{k',ij}] = \alpha[\ba_\eta^\top\quad -b_\eta]$ for some $\alpha \in \{-1,0\}$, and the activation $\lambda_{k',ij}$ should be set to ensure this holds.
\end{enumerate}
\end{theorem}

The proof for Theorem \ref{thm: activation_flipping} is in the Appendix. Then, the procedure to generate a neighboring activation pattern is defined in Algorithm \ref{Algo:neighbor_ap}. We present Algorithm \ref{Algo:neighbor_ap} in as simple a manner as possible, however, for a more efficient approach, one need not loop over all neurons in the network, but only those that meet condition 2) of Theorem \ref{thm: activation_flipping}.



\begin{algorithm}[!h]
\SetAlgoLined
\DontPrintSemicolon
 \KwIn{$\lambda_k$, $\mathbf{a}_\eta$, $b_\eta$, $\bTheta$}
 \KwOut{$\lambda_{k'}$}
 $\lambda_{k'} \leftarrow  \lambda_k$ \algorithmiccomment{Initialize neighbor region $\lambda$} \\
 
\For{$i = 1,\ldots,L-1$, $j = 1,\ldots,l_i$}{ 
$\mathbf{a}_{k',ij},\ b_{k',ij} \leftarrow$ (\ref{eq:normalize}) \\
\uIf{$\mathbf{a}_{k',ij} = \mathbf{a}_\eta \quad \textbf{\upshape and}\quad b_{k',ij} = b_\eta$}
{$\lambda_{k',ij} \leftarrow \neg \lambda_{k,ij} $}
\uElseIf{$\mathbf{a}_{k',ij} = \mathbf{0} \quad \textbf{\upshape and}\quad b_{k',ij} = 0$}
{$\lambda_{k',ij} \leftarrow 0$}
}
\Return{$\lambda_{k'}$}
\caption{Neighboring Activation Pattern}
\label{Algo:neighbor_ap}
\end{algorithm}

\subsection{Reachable Polyhedral Marching} \label{Cell Enumeration}

Building on Algorithms \ref{Algo:essential hrep} \& \ref{Algo:neighbor_ap}, we now define our main algorithm, Reachable Polyhedral Marching (RPM) in Algorithm \ref{Algo:cell_enumeration} for explicitly enumerating all polyhedra in the input space of a ReLU network, resulting in the explicit PWA representation. 
First we start with a polyhedral set of inputs we would like to enumerate affine regions over.
This set of inputs can also be $\mathbb{R}^n$, but in most applications it makes sense to consider a bounded set of inputs.
Then, we find an initial point within the input set and evaluate the network to find the activation pattern. 
From this, the essential H-representation is found using Algorithm \ref{Algo:essential hrep}. 
For each essential constraint we generate a neighboring activation pattern using Algorithm \ref{Algo:neighbor_ap}. 
The process then repeats with each new neighboring activation pattern being added to a working set. 
For a given starting polyhedron each neighbor polyhedron is enumerated, and since all polyhedra are connected via neighbors, Algorithm \ref{Algo:cell_enumeration} is guaranteed to enumerate every polyhedron in the input space. 
Figure \ref{fig:alg_steps} illustrates the steps of RPM and Figure \ref{fig:cell_enum} shows the result of applying Algorithm \ref{Algo:cell_enumeration} to a ReLU network trained to model a quadratic function.

The bulk of the computation for Algorithm \ref{Algo:cell_enumeration} is dedicated to solving the LPs needed to retrieve the essential H-representation for each polyhedron (Algorithm \ref{Algo:essential hrep}).
In the worst case, for each polyhedron in the input space, Algorithm \ref{Algo:cell_enumeration} solves an LP for each neuron in the network. 
As noted in Section \ref{sec:relu_networks}, the number of polyhedra for the algorithm to explore scales according to $\mathcal{O}(N^n)$ where $N$ is the number of neurons and $n$ is the dimension of the input space.
Thus, in the worst case, the number of LPs solved by the algorithm scales according to $\mathcal{O}(N^{2n})$.
This exponential complexity in the input dimension is unsurprising and is inherent to the exact analysis of ReLU neural networks \cite{reluplex}.

\begin{algorithm}
\SetAlgoLined
\DontPrintSemicolon
 \KwIn{$\lambda_{0}$, $\bTheta$}
 \KwOut{Explicit PWA representation}
 PWA = $\emptyset$;  visited = $\emptyset$ \\
 working set = \{$\lambda_0$\} \\
 \While{$\text{\upshape working set} \neq \emptyset$}{
 $\lambda_k \leftarrow$ pop next $\lambda$ off working set \\
 $\bA_k, \bb_k \leftarrow$ Algorithm \ref{Algo:essential hrep} \algorithmiccomment{Retrieve essential H-rep} \\
 $\mathbf{C}_k,\ \mathbf{d}_k \leftarrow$ (\ref{eq:CdDef}) \algorithmiccomment{Retrieve affine map} \\
 push ($\bA_k, \bb_k, \bC_k, \bd_k$) onto PWA \\
 \For(\algorithmiccomment{For each neighbor}){$\mathbf{a}_{k,i}, b_{k,i} \in \bA_k, \bb_k$}{
 $\lambda_{k'} \leftarrow$ Algorithm \ref{Algo:neighbor_ap} \algorithmiccomment{$\ba_\eta, b_\eta = \ba_{k,i}, b_{k,i}$} \\
 \If{$\lambda_{k'} \notin \text{\upshape visited} \cup \text{\upshape working set}$}{
 push $\lambda_{k'}$ onto working set \\
 }
 }
 push $\lambda_k$ onto visited\\
 }
 \Return{\text{\upshape PWA}}
 \caption{Reachable Polyhedral Marching}
 \label{Algo:cell_enumeration}
\end{algorithm}

\section{Reachability} 
\label{Sec:Reachability}
\subsection{Forward Reachability} \label{Forward}
The forward reachable set of a PWA function over some input set is simply the union over the forward reachable sets of each polyhedron comprising the input set. The image of a polyhedron under an affine map is
\begin{align}
    P_{image} = \{\by \mid \by=\mathbf{Cx}+\mathbf{d},\quad \mathbf{Ax} \le \mathbf{b}\}. \label{eq:forward reach}
\end{align}
For invertible $\mathbf{C}$, the H-representation of the image is
\begin{align}
    P_{image} = \{\by \mid \mathbf{AC}^{-1}\by \le \mathbf{b} + \mathbf{AC}^{-1} \mathbf{d}\}. \label{eq:image}
\end{align}
In the case the affine map is not invertible, more general polyhedral projection methods such as block elimination, Fourier-Motzkin elimination, or parametric linear programming can compute the H-representation of the image \cite{cdd_manual, mpLP}. 
Our implementation uses the block elimination projection in the case of a non-invertible affine map \cite{cdd}.

Our RPM algorithm is used to perform forward reachability as follows. We first specify a polyhedral input set whose image through the ReLU network we want to compute.  This is a set over which to perform the RPM algorithm. For each activation pattern $\lambda_k$ we also compute the image of $(\bA_k, \bb_k)$ under the map $\by = \mathbf{C}_k\bx + \mathbf{d}_k$. To do this, an additional line is introduced between lines 9 and 10 of Algorithm \ref{Algo:cell_enumeration}
\begin{align}
    (\bA_k, \bd_k)_{forward} \leftarrow \text{project}(\bA_k, \bd_k, \mathbf{C}_k, \mathbf{d}_k)
\end{align}
where $\text{project}(\bA, \bb, \bC, \bd)$ applies (\ref{eq:image}) if $\mathbf{C}$ invertible and block elimination on (\ref{eq:forward reach}) otherwise.

\subsection{Backward Reachability} \label{Backward}
The preimage of a polyhedron under an affine map is
\begin{subequations}
    \begin{align}
        P_{pre} &= \{\bx \mid \by=\mathbf{Cx}+\mathbf{d},\quad \mathbf{Ay} \le \mathbf{b}\} \\
        &= \{\bx \mid \mathbf{ACx} \le \mathbf{b}-\mathbf{Ad}\}. \label{eq:backward reach}
    \end{align}
\end{subequations}
Like forward reachability, performing backward reachability only requires a small modification to Algorithm \ref{Algo:cell_enumeration}. We first specify a polyhedral output set whose preimage we would like to compute.  For each activation pattern $\lambda_k$, we then compute the intersection of the preimage of the given output set under the map $\mathbf{C}_k\bx + \mathbf{d}_k$ with $P_k$ (the polyhedron for which the affine map is valid). Two additional lines are thus introduced between lines 9 and 10 of Algorithm \ref{Algo:cell_enumeration}
\begin{subequations}
    \begin{gather}
    P_{pre} \leftarrow \text{Equation}\ \ref{eq:backward reach} \label{eq:p_in} \\
    (\bA_k, \bb_k)_{backward} \leftarrow P_k \cap P_{pre}. \label{eq:h_back}
    \end{gather}
\end{subequations}
Multiple backward reachable sets can be solved for simultaneously at the added cost of repeating (\ref{eq:p_in}) and (\ref{eq:h_back}) for each output set argument to Algorithm \ref{Algo:cell_enumeration}.

Finally, we address the issue of finding $t$-step forward and backward reachable sets for a dynamical system represented as a neural network, $\bx_{t+1} = F(\bx)$.
For a ReLU network that is applied iteratively over $t$ time steps, $\bx_t = F \circ \cdots \circ F(\bx_0)$.
Thus, we can concatenate the ReLU network $F$ with itself $t$ times, forming one larger network representing a $t$-step update.
If the original network $F(\bx)$ has $r$ hidden layer neurons and $N$ layers, the concatenated network has $tr$ neurons and $t(N-1) + 1$ layers.  
We then simply perform RPM for forward or backward reachability on the concatenated network.
Thus, computing a $t$-step reachable set is as expensive as computing a $1$-step reachable set for a network with $t$ times the depth.


\section{Invariant Sets}
\label{Sec:ROA}
\subsection{Control Invariant Sets}
In this subsection we discuss a strategy for computing control invariant sets once we have the PWA representation of a dynamical system $\bx_{t+1} = F(\bx_t, \mathbf{u}_t)$.
This strategy for computing control invariant sets is implemented in MPT3 specifically for PWA dynamical systems \cite{MPT3}.
First, we note that for some state transition functions $F$, there may not exist a control invariant set.
Existence depends on the state domain $\mathcal{X}$ and the set of admissible control inputs $\mathcal{U}$, as well as the effect of control inputs on the state transitions given by $F$.
If a control invariant set exists, it can be computed via the following recursion which uses backward reachability,
\begin{subequations}
\begin{gather}
    X_0 = \mathcal{X} \\
    X_{i+1} = \{\bx \mid F(\bx,\mathbf{u}) \in X_{i}, \mathbf{u} \in \mathcal{U} \},
\end{gather}
\end{subequations}
for a compact polyhedral state domain $\mathcal{X}$ and set of admissible inputs $\mathcal{U}$. 
If $X_{i+1} = X_i$, then $X_i$ is control invariant.

\subsection{Regions of Attraction}
In this subsection we consider the problem of computing invariant ROAs given a ReLU network that represents an autonomous dynamical system $\bx_{t+1} = F(\bx_t)$. 
One of the most common goals in designing a control policy is to stabilize the system around some equilibrium.
By identifying a ROA, we are identifying a set of states which is stabilized.

The strategy we use to compute ROAs is as follows.
We first use RPM to enumerate each affine region of $F$ and within each region check for the existence of a stable fixed point.
If one is found, we then compute a small `seed ROA' around this equilibrium. 
Then we use backward reachability to find the $t$-step backward reachable set of the seed ROA, which itself is guaranteed to be an ROA (see \cite{pwa_lygeros}). 
The volume of the resulting ROA grows as more backward reachability steps are taken. Thus, the maximal ROA is better approximated by using more backward reachability steps.

In the following subsections we detail the procedure for finding fixed points, checking whether they are stable equilibria, and computing seed ROAs. 
We also describe how the computation of the ROA can be accelerated by exploiting the connected walk of RPM, and under what conditions this accelerated method captures the same result as the original method.
Finally, as with control invariant sets, there are some dynamical systems for which no ROA exists.
However, in our outlined procedure, the time taken to find fixed points (if they exist) is a small fraction of the time taken to find the ROA through backward reachability.
Thus, verifying that no ROA exists can be accomplished with relative ease (compared to verifying that no control invariant set exists, for example).

\subsubsection{Finding Fixed Points}
To find fixed points of a ReLU network we use RPM to solve for the explicit PWA representation and for each region $k$ solve the following system of equations for $\bx$.
\begin{align}
    \mathbf{C}_k\bx + \mathbf{d}_k &= \bx \label{eq:fp}
\end{align}
We consider $\bx$ a fixed point if it is the unique solution of (\ref{eq:fp}) and lies on the interior of $P_k$. 
A fixed point in $P_k$ is a stable equilibrium if $\mathbf{C}_k$ is a stable dynamics matrix. 
That is, if the eigenvalues of $\mathbf{C}_k$ have magnitude strictly less than one.
When all eigenvalues have magnitude strictly less than one, trajectories near the fixed point will asymptotically converge.

\subsubsection{Finding Seed ROAs} \label{seed}
For our backward reachability algorithm we require a polyhedral seed ROA to start from. Once a stable fixed point is located we translate the corresponding affine system to a linear one via a coordinate transformation. We can then use existing techniques for computing invariant polyhedral ROAs of stable linear systems \cite{hennet1995discrete}. 
Once a polyhedral ROA is found it can be scaled it up or down to ensure it is a subset of the polyhedron for which the local affine system is valid.
Then, it is well known that the backward reachable set of an invariant set is also an invariant set \cite{pwa_lygeros}.
Thus by finding the backward reachable set of the seed ROA, we grow the ROA.

\subsection{An Accelerated Procedure} \label{sec:homeo}
The backward reachability algorithm as previously presented can be made more efficient if RPM is restricted to only enumerated a connected backward reachable set.
Knowing that the fixed point $\bx_{eq}$ must be in the ROA, initializing RPM at the fixed point, and then only enumerating a connected backward reachable set of the seed ROA can significantly reduce the computation time.
In this case, we can alter the backward reachability algorithm to only add neighbor polyhedra to the working set if they are neighbors of a polyhedron known to be in the backward reachable set. 

The resulting set of states found by this procedure still converges to the equilibrium, but since we may not be computing the entire backward reachable set of the seed ROA, it may not be an invariant set.
Next we consider sufficient conditions for this accelerated procedure to still produce an ROA which is an invariant set.

It is known that the backward reachable set of a connected set will also be connected if the function mapping between the input and output spaces is a homeomorphism.
\begin{definition}[Homeomorphism] \label{def:homeo}
A function $f: \mathcal{X} \rightarrow \mathcal{Y}$ between two metric spaces is a homeomorphism if (i) $f$ is bijective, (ii) $f$ is continuous, (iii) $f^{-1}$ is continuous.
Furthermore, the composition of homeomorphisms is also a homeomorphism.
\end{definition}
PWA homeomorphic conditions are given in \cite{homeo2}.
\begin{theorem}[Homeomorphic PWA Functions \cite{homeo2}]
For affine maps $\mathbf{C}_k\bx + \mathbf{d}_k$ of regions $P_k \in \mathcal{P}$, let $\mathbf{C}[r]$ denote the matrix consisting of the first $r$ rows and columns of $\mathbf{C}$. If for each $r = 1, \ldots, n$,
\begin{align}
   \text{sgn det}(\mathbf{C}_k[r]) = \sigma_r \quad \forall k
\end{align}
where $\sigma_r \neq 0$ and sgn det() is the sign of the determinant, then the PWA function is a homeomorphism. 
\end{theorem}
Note that all $\mathbf{C}_k$ nonsingular is a necessary but insufficient condition. 
This theorem gives us a simple way to check whether a ReLU network is a homeomorphism once we have obtained the explicitly defined PWA function. 

Thus, if the ReLU network defining the dynamical system is a homeomorphism, then restricting RPM to enumerate a connected backward reachable set of the seed ROA results in no loss of generality and the set obtained remains invariant.
In our examples we explore the effect of this accelerated procedure and show that it can lead to dramatic decreases in computation time.
The computation time saved is roughly the ratio between the number of polyhedra in the connected backward reachable set and the number of polyhedra over the entire state domain $\mathcal{X}$.
Depending on the definition of the state domain, this ratio can be quite extreme.

Our method of decomposing ReLU networks into their explicit PWA representations using RPM and checking the homeomorphism condition is, to the best of our knowledge, the first approach for determining the invertibility of an arbitrary ReLU network. 
This procedure for checking invertibility may find use beyond the problems studied in this paper, like in training normalizing flow models~\cite{papamakarios2021normalizing}.



\section{Examples} 
\label{Sec:Examples}
All examples are run on a desktop computer with AMD Ryzen 5 5600X 3.7 GHz 6-Core processor and 16GB of RAM.
Our implementation uses the GLPK open-source LP solver and the JuMP optimization package \cite{DunningHuchetteLubin2017}. 
The polyhedron coloring in plots is random.


Given that answering queries about exact reachable sets of ReLU neural networks is NP-complete \cite{reluplex}, the application of the methods presented in this paper is constrained by (i) the size of the neural network to analyze and (ii) the properties one wishes to verify.
When only a single-step reachability computation is needed, as in verifying open-loop safety properties of a learned policy, the ReLU networks can be moderately sized.
With multi-step reachability computations, such as in Examples \ref{sec:vanderpol}, \ref{Pendulum}, \ref{Taxinet}, the ReLU networks are smaller since the network is concatenated with itself for each desired reachability step.

\subsection{van der Pol Oscillator Example} \label{sec:vanderpol}
In this example we illustrate ROA computation for a learned reverse-time van der Pol oscillator.
These dynamics give rise to a bounded, nonconvex ROA (the boundary is the limit cycle). 
We train a ReLU network to learn a 10 Hz discrete-time model from the continuous-time model,
\begin{subequations}
\begin{align}
    \dot{x}_1 &= -x_2 \\
    \dot{x}_2 &= x_1 + x_2(x_1^2 - 1).
\end{align}
\end{subequations}

We then find a seed ROA for the stable affine system as described in Section \ref{seed}. Next, we verify the network is a homeomorphism by evaluating the conditions in Section \ref{sec:homeo}; from this we know that a connected set in the output space will have a connected preimage. 
\textcolor{black}{We repeat the network training and these preprocessing steps $10$ times and provide ranges for the number of affine regions and computation time for each step in Table \ref{tab:van_details}.
We run the longer, backward reachability computation, on only one of the networks (a network with 248 regions).}


\begingroup
\setlength{\tabcolsep}{5pt}
\begin{table}[h]
    \centering
    \caption{van der Pol Dynamics Analysis}
    {\color{black}\begin{tabular}{c | c}
        Layer Sizes & $2$, $20$, $20$, $2$ \\
        Number of Affine Regions & \textcolor{black}{$93 - 258$} Regions \\
        Time to Enumerate Affine Regions & \textcolor{black}{$1.93 - 3.05$} Seconds \\
        Time to Check if Homeomorphism & \textcolor{black}{$0.04 - 0.04$} Seconds \\
        Time to Find Fixed Point & \textcolor{black}{$0.27 - 0.33$} Seconds \\
        Time to Find Seed ROA & \textcolor{black}{$0.74 - 2.2$} Seconds
    \end{tabular}}
    \label{tab:van_details}
\end{table} 
\endgroup

Then, to find a $t$-step ROA, we simply perform backward reachability on the seed ROA for the desired number of steps.
We show the resulting 5, 10, 15, 20, 25, and 30-step ROAs in Figure \ref{fig:van_panel}. 
Note that in this example we compute a $t$-step ROA without needing intermediate ROAs by simply concatenating $t$ dynamics networks. 
However, we show multiple ROAs for illustration. By verifying the learned dynamics function is homeomorphic, we can run RPM on a restricted domain, providing a significant computational speedup. 
In Table \ref{tab:van_der_pol} we report the number of regions in the resulting ROAs as well as the number of regions explored by the backward reachability algorithm with and without leveraging the homeomorphic property of the network. 
The results in the table indicate a $15$x speedup is gained by leveraging the knowledge that the ROA will be connected since the seed ROA is connected and the network is homeomorphic.

\textcolor{black}{In addition, Table \ref{tab:van_der_pol} demonstrates that using RPM with knowledge of the homeomorphism also provides significant performance gains over a state of the art exact reachability algorithm \cite{yang2020reachability}.
Although the approach from \cite{yang2020reachability} is clearly faster than the ``naive" RPM implementation, this is not the case when using the homeomorphic implementation.
Moreover, the state bounds were chosen with prior knowledge of the true ROA, and if we relax the state bounds to $x_1 \in [-10, 10]$, $x_2 \in [-10, 10]$, then the method from \cite{yang2020reachability} enumerates the 15-step ROA in 610 seconds (scaling with the size of the state space), whereas the homeomorphic RPM implementation does not scale with the size of the state space (but rather with the size of the ROA), and thus the solve time remains 157 seconds.
}

\begingroup
\setlength{\tabcolsep}{2pt}
\begin{table}[h]
    \centering
    \caption{van der Pol ROA}
    {\fontsize{8}{10}\selectfont 
    \begin{tabular}{c c c c c c c}
        & &\multicolumn{2}{c}{\textbf{Homeomorphic}}&\multicolumn{3}{c}{\textbf{Naive}} \\
        \cmidrule(r){3-4}  \cmidrule(r){5-7}
        Steps& \# in ROA &\# Explored&Time (s) &\# Explored&Time-RPM&\textcolor{black}{Time-\cite{yang2020reachability}}\\ [0.5ex] 
        \hline \\ [-1.5ex]
        5    &$500$&$594$&$5$&$7362$&$98$&\textcolor{black}{8} \\
        10   &$2525$&$2735$&$42$&$33,946$&$627$&\textcolor{black}{127} \\
        15   &$6569$&$6898$&$157$&$85,217$&$2280$&\textcolor{black}{295} \\
        20   &$12,553$&$12,995$&$413$&$-$&$-$&$-$ \\
        25   &$20,949$&$21,554$&$954$&$-$&$-$&$-$ \\
        30   &$32,753$&$33,500$&$1956$&$-$&$-$&$-$
    \end{tabular}}
    \label{tab:van_der_pol}
\end{table} 
\endgroup

In contrast, the standard approach for finding an ROA via a Lyapunov function (see \cite{biswas}) does not return a solution. The Lyapunov approach first finds an invariant set (that is as large as possible). Then, uses convex optimization to synthesize a piecewise Lyapunov function (if one exists). We use MPT3 to carry out this approach. MPT3 first finds an invariant set by computing the following recursion until convergence,
\begin{subequations}
\begin{align}
    X_0 &= \mathcal{X} \\
    X_{i+1} &= \{\bx \mid f(\bx) \in X_{i} \}
\end{align}
\end{subequations}
for a compact polyhedral domain $\mathcal{X}$, which in this case is the rectangle $x_1 \in [-2.5, 2.5]$, $x_2 \in [-3, 3]$. After running this recursion in MPT3 for 200 steps (9 hours), an invariant set was not found and we terminated the procedure. Our backward reachability approach, leveraging the homeomorphic property of the dynamics, is much more reliable for finding ROAs. Furthermore, our approach allows the user to trade-off between computation time and size of the ROA, whereas the standard Lyapunov approach is ``all or nothing".

\begin{figure}[ht]
\includegraphics[width=0.98\linewidth]{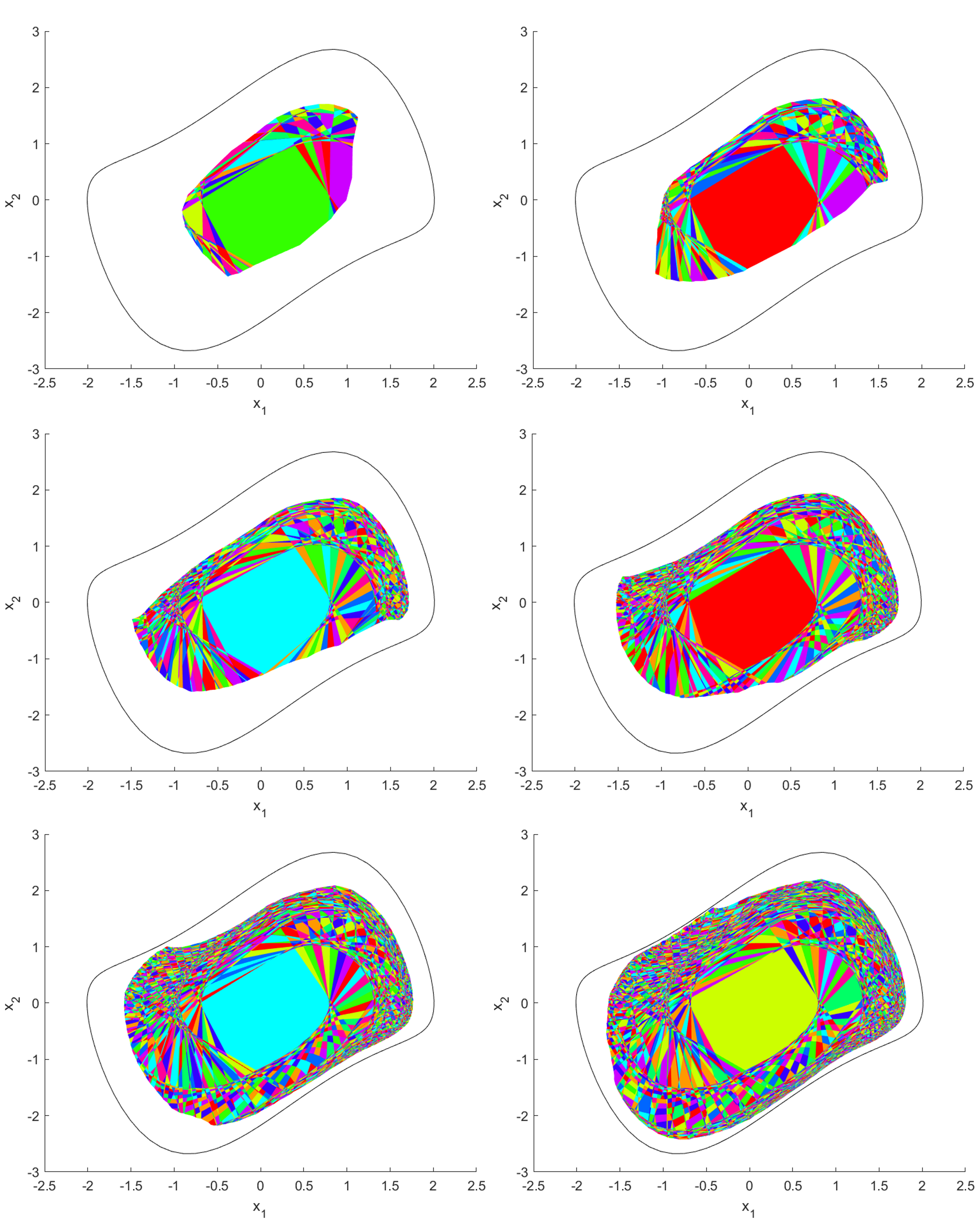} 
\setlength\belowcaptionskip{-1.\baselineskip}
\centering
\caption{5, 10, 15, 20, 25, and 30-step regions of attraction (ROAs) for the learned van der Pol system. The learned dynamics network is continuously invertible, implying that preimages of connected sets are connected. This property restricts the space explored with RPM, resulting in a 15x speedup in ROA computation. The black line represents the boundary of the true continuous-time maximal ROA.}
\label{fig:van_panel}
\end{figure}

\subsection{Pendulum Example} \label{Pendulum}
In this example we pair our RPM tool with the popular multi-parametric toolbox (MPT3) \cite{MPT3} to find a control invariant set for a learned pendulum model. We first learn a ReLU network that models the continuous-time dynamics
\begin{align}
    \Ddot{\theta} = \frac{u + mgl\sin{\theta}}{ml^2}
\end{align}
as a discrete-time model where $m = 1$kg, $g = 9.81$m/s$^{2}$, $l=1$m, and $\theta$ is measured clockwise from the upright position. 
\textcolor{black}{We repeat the network training and enumeration of affine regions $10$ times and provide ranges for the number of affine regions and computation time in Table \ref{tab:pend_details}.
We run the longer, invariant set computation, on only one of the networks (a network with 217 regions).}

Details of the learned dynamics network are in Table \ref{tab:pend_details}.


\begingroup
\setlength{\tabcolsep}{5pt}
\begin{table}[h]
    \centering
    \caption{Pendulum Dynamics Analysis}
    {\color{black}\begin{tabular}{c | c}
        Layer Sizes & $3$, $16$, $16$, $2$ \\
        Number of Affine Regions & \textcolor{black}{$217 - 823$} Regions \\
        Time to Enumerate Affine Regions & \textcolor{black}{$2.35 - 4.11$} Seconds
    \end{tabular}}
    \label{tab:pend_details}
\end{table} 
\endgroup

We then use MPT3 to compute the control invariant set shown in Figure \ref{fig:ctrl_inv} with computational details in Table \ref{tab:pend_set}.

\begingroup
\setlength{\tabcolsep}{5pt}
\begin{table}[h]
    \centering
    \caption{Pendulum Control Invariant Set}
    {\color{black}\begin{tabular}{c c c}
        Algorithm Steps&\# Polyhedra in Set&Time\\
        \hline \\ [-1.5ex]
        $61$ & $958$ & $5.4$ hours
    \end{tabular}}
    \label{tab:pend_set}
\end{table} 
\endgroup

As shown in Figure \ref{fig:ctrl_inv}, the control invariant set is the union of 3 disjoint sets. The leftmost set represents the states for which there is enough control authority to keep the pendulum clockwise of the swing-down position, but not enough to swing back up to upright. The middle set represents the states for which there is enough control authority to keep the pendulum from swinging all the way down. Finally, rightmost set has a similar interpretation to the leftmost set, being the states for which there is enough control authority to keep the pendulum counterclockwise of the swing-down position, but not enough to swing back upright. With the explicit PWA representation given by RPM, the wide array of PWA analysis tools offered by the MPT3 toolbox is able to be leveraged for learned dynamics and controls models.
\begin{figure}[h]
\includegraphics[width=0.85\columnwidth]{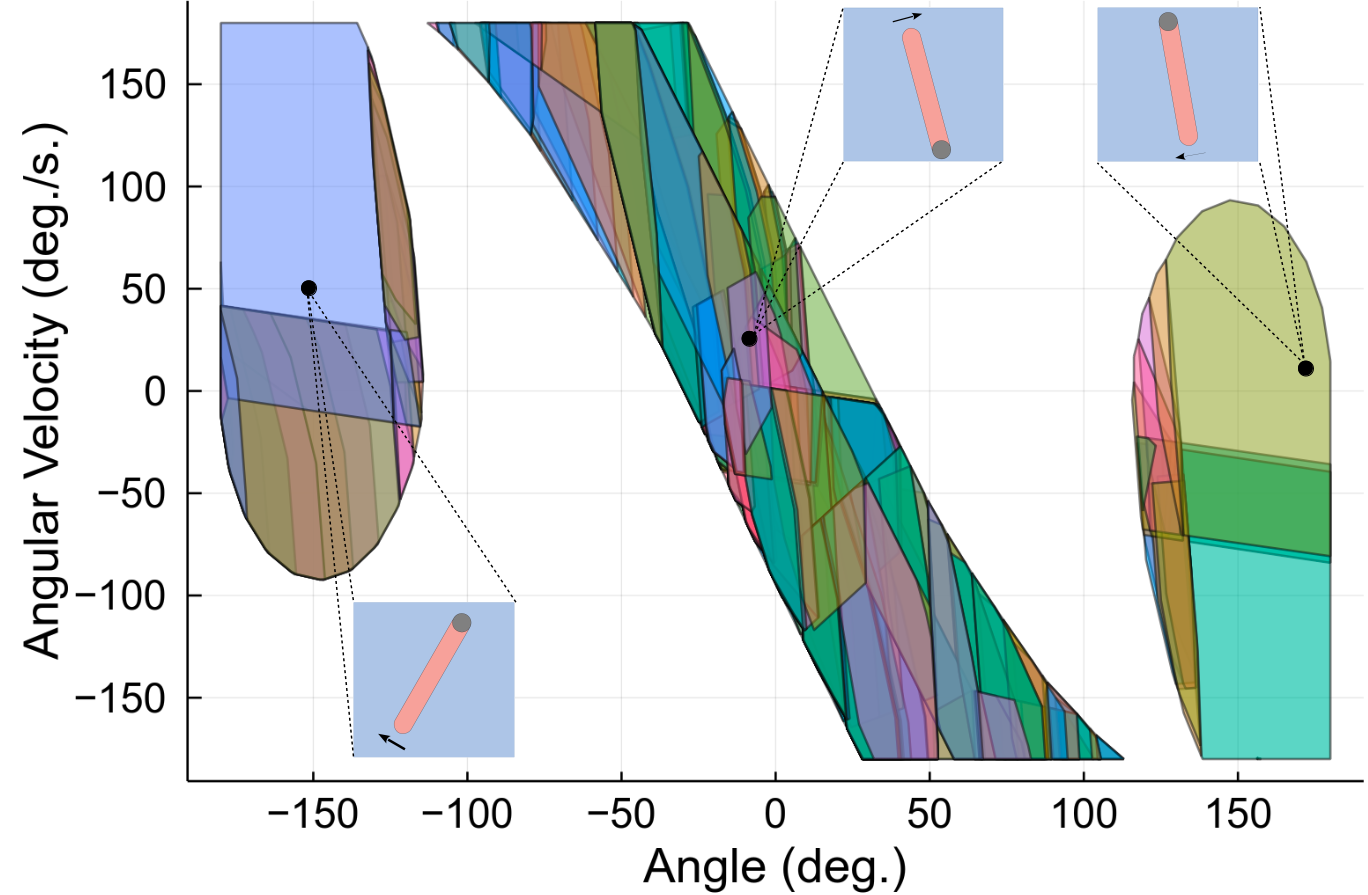} 
\centering
\caption{A control invariant set for a learned, torque controlled pendulum system. The domain of the dynamics is $\theta \in (-180\degree, 180\degree)$, $\dot{\theta} \in [-180 \degree/\text{s}, 180 \degree/\text{s}]$, $u \in [-5, 5]$ N-m,  where $0\degree$ is upright. Note that we restrict the pendulum so it is not allowed to pass through the down position.  When RPM is paired with MPT3 we are able to compute complicated control invariant sets of learned systems, like this unconnected set, as opposed to convex approximations. Within each connected component a sample state is illustrated to help visualize the possible pendulum configurations.}
\label{fig:ctrl_inv}
\end{figure}

\subsection{Verifying Image-Based Control} \label{Taxinet}
\begin{figure}
    \setlength\abovecaptionskip{1.2\baselineskip}
    \setlength\belowcaptionskip{-0.8\baselineskip}
     \centering
     \begin{subfigure}[b]{0.48\columnwidth}
         \centering
         \includegraphics[width=\textwidth]{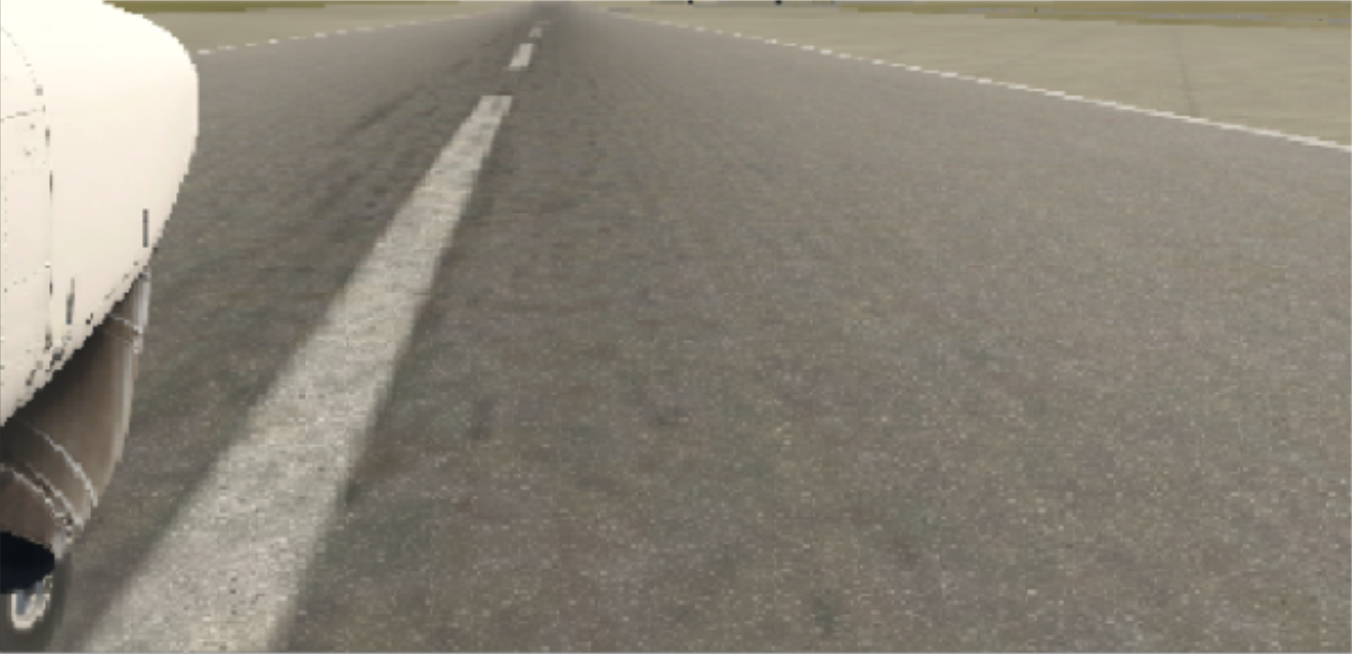}
         \caption{Full observation.}
         \label{fig:taxinet_full}
     \end{subfigure}
     \hfill
     \begin{subfigure}[b]{0.48\columnwidth}
         \centering
         \includegraphics[width=\textwidth]{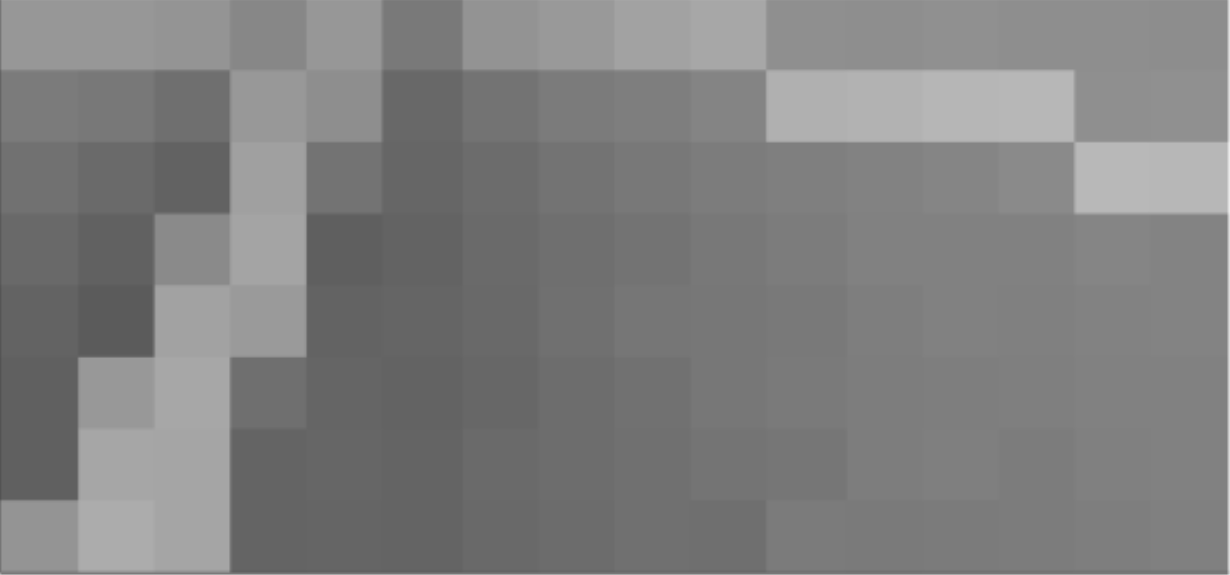}
         \caption{Downsampled observation.}
         \label{fig:taxinet_coarse.}
     \end{subfigure}
        \caption{The X-Plane 11 flight simulator is used to generate realistic image observations (left) that are then downsampled (right). The downsampled images are used to learn an image-based controller, as well as to learn an observation model from state to image observation. Putting these two components together along with a learned dynamics network results in closed-loop system dynamics represented as a single ReLU network. These images and learned controller/observation models are from \cite{verify_gan}.}
        \label{fig:taxinet_obs}
\end{figure}

lIn this example we apply the tools of Section \ref{Sec:ROA} to a larger dynamical system that models the closed-loop behavior of a taxiing airplane controlled from camera images.
Our goal is to investigate the set of states that are stabilized by a learned image-based controller by modeling the closed-loop dynamical system with a single ReLU network, and applying the accelerated procedure from Section \ref{Sec:ROA}. 

The specific system we consider was introduced in \cite{verify_gan} where a taxiing airplane is regulated to the centerline of a runway via rudder-angle control input based on images from an onboard camera. Because of the image observation, we cannot model the observation from physical principles. Instead, we modify the learned models in \cite{verify_gan} to obtain a network that approximates the observed image given a state
\begin{align}
    \mathbf{o}_t = F_{obs}(\bx_t),
\end{align}
and a control network that produces a rudder angle command given an observed image
\begin{align}
    u = F_{ctrl}(\mathbf{o}_t).
\end{align}
The continuous time dynamics are
\begin{subequations}
\begin{align}
    \dot{p} &= v\sin{\theta} \\
    \dot{\theta} &= \frac{v}{L}\tan{u}
\end{align}
\end{subequations}
for crosstrack position $p$, heading error $\theta$, rudder angle $u$, and speed $v=5$ m/s. We learn a 1Hz discrete-time model,
\begin{align}
    \bx_{t+1} = F_{dyn}(\bx_t, u),
\end{align}
and then have the closed-loop system
\begin{align}
    \bx_{t+1} = F_{cl}(\bx_t) = F_{dyn}(\bx_t, F_{ctrl} \circ F_{obs}(\bx_t))
\end{align}
which can be represented as one ReLU network.

Next, we use Algorithm \ref{Algo:cell_enumeration} to retrieve the explicit PWA function for $F_{cl}$ over the domain $p \in [-5, 5]$ m., $\theta \in [-15^\circ, 15^\circ]$. 
The resulting PWA function has 244,920 regions, took 3 hours to compute, and was not found to be a homeomorphism.
Timing for these steps is summarized in Table \ref{tab:taxinet_details}.
Although the function is not homeomorphic, when applying the accelerated algorithm from Section \ref{sec:homeo}, we are still guaranteed that all states within this set asymptotically converge to the stable fixed point.


\begingroup
\setlength{\tabcolsep}{5pt}
\begin{table}[h]
    \centering
    \caption{Image-Based Control Dynamics Analysis}
    {\color{black}\begin{tabular}{c | c}
        Layer Sizes & \begin{tabular}{@{}c@{}}$2$, $260$, $260$, $260$, $260$, \\ $20$, $12$, $12$, $16$, $2$\end{tabular} 
          \\
        Number of Affine Regions & $244,920$ Regions \\
        Time to Enumerate Affine Regions & $3$ Hours \\
        Time to Check if Homeomorphism & $0.04$ Seconds \\
        Time to Find Fixed Point & $2$ Seconds \\
        Time to Find Seed ROA & $3$ Seconds
    \end{tabular}}
    \label{tab:taxinet_details}
\end{table} 
\endgroup

This PWA function is significantly larger than other PWA dynamics that have been studied in the literature ($\sim 250,000$ regions vs $\sim 1,000$ regions). 
After computing the explicit PWA function, we find a stable fixed point. Next, we find a seed ROA for each stable fixed point and perform backward reachability to grow the ROAs. During the backward reachability computation, we use a few strategies to make the problem more tractable. First, we restrict the backward reachability algorithm to only explore connected backward reachable sets as described in Section \ref{sec:homeo}. 
Second, to avoid numerical concerns accompanied with very small polyhedra, we compute backward reachable sets iteratively with the 1-step PWA dynamics directly instead of in one shot using concatenated ReLU networks. 
Third, at each step we merge the polyhedra into a union of overlapping polyhedra to reduce the complexity of the set representation.

In Figure \ref{fig:taxi_roa} we show the set of states that our algorithm was able to show is asymptotically stabilized by the learned controller.
This set was computed using $11$ backward reachable steps (11 seconds of trajectory evolution), and is represented as a union of $28,715$ polyhedra.

\begingroup
\setlength{\tabcolsep}{5pt}
\begin{table}[h]
    \centering
    \caption{Image-Based Control ROA}
    {\color{black}\begin{tabular}{c c c}
        Algorithm Steps&\# Polyhedra in Set&Time\\
        \hline \\ [-1.5ex]
        $11$ & $28,715$ & $14$ hours
    \end{tabular}}
    \label{tab:taxinet_roa}
\end{table} 
\endgroup

\begin{figure}
\includegraphics[width = 0.85\linewidth]{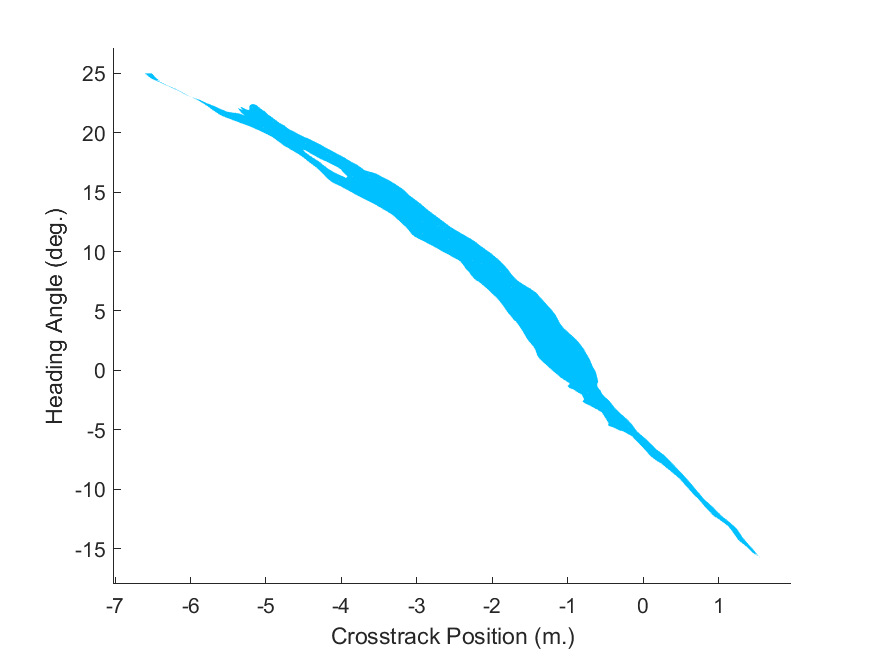} 
\setlength\belowcaptionskip{-1.\baselineskip}
\centering
\caption{Set of states that our approach finds are stabilized by the learned controller for the image-based airplane runway taxiing system from Fig.~\ref{fig:taxinet_obs}.
This set is the 11-step backward reachable set for the seed ROA found around the stable fixed point.
This set of states is represented by 28,715 polyhedra.
}
\label{fig:taxi_roa}
\end{figure}

\section{Conclusions}
\label{Sec:Conclusions}
We proposed the Reachable Polyhedral Marching (RPM) algorithm to incrementally construct an exact PWA representation of a ReLU network. 
RPM computes an explicit PWA function representation for a given ReLU network that can then be used to quickly find forward and backward reachable sets of ReLU networks. 
In addition, we demonstrated how, for dynamical systems modeled as ReLU networks, RPM can be leveraged to (i) compute intricate control invariant sets and (ii) certify stability and compute ROAs. 
RPM enumerates affine regions and reachable sets in an incremental and connected fashion, which we show can lead to massive acceleration when computing ROAs.
Lastly, RPM also provides a novel way to determine invertibility of ReLU networks.

In future work we are interested in investigating the homeomorphic properties of ReLU networks. 
There are many studies on the invertibility of neural networks that propose to restrict the network architecture and training such that invertibility is guaranteed with high probability. 
However, these restrictions might be loosened by further understanding the connection between PWA homeomorphisms and ReLU networks. Another area for future work is to use the insights gained from Theorem \ref{thm: activation_flipping} to develop an inverse RPM method for encoding explicit PWA functions as ReLU networks. 
Existing work in this direction yields ReLU networks with many more parameters than necessary, not taking full advantage of the dependency of PWA parameters. Also of interest is the potential to use the PWA decomposition given by RPM to assess where in the domain the ReLU network is likely to generalize well and where it is not.


\section*{Acknowledgements}
The authors would like to thank Yue Meng for helpful insights and feedback on the RPM algorithm.





\bibliographystyle{./IEEEtran} 
\bibliography{./IEEEabrv,./IEEEexample}

\begin{thebibliography}{10}
\providecommand{\url}[1]{#1}
\csname url@rmstyle\endcsname
\providecommand{\newblock}{\relax}
\providecommand{\bibinfo}[2]{#2}
\providecommand\BIBentrySTDinterwordspacing{\spaceskip=0pt\relax}
\providecommand\BIBentryALTinterwordstretchfactor{4}
\providecommand\BIBentryALTinterwordspacing{\spaceskip=\fontdimen2\font plus
\BIBentryALTinterwordstretchfactor\fontdimen3\font minus \fontdimen4\font\relax}
\providecommand\BIBforeignlanguage[2]{{%
\expandafter\ifx\csname l@#1\endcsname\relax
\typeout{** WARNING: IEEEtran.bst: No hyphenation pattern has been}%
\typeout{** loaded for the language `#1'. Using the pattern for}%
\typeout{** the default language instead.}%
\else
\language=\csname l@#1\endcsname
\fi
#2}}

\bibitem{blanchini2008set}
F.~Blanchini, S.~Miani, \emph{et~al.}, \emph{\href{https://link.springer.com/book/10.1007/978-3-319-17933-9}{Set-Theoretic Methods in Control}}.\hskip 1em plus 0.5em minus 0.4em\relax Springer, 2008, vol.~78.

\bibitem{HJBFastMarching}
J.~A. Sethian, ``\href{https://www.pnas.org/doi/abs/10.1073/pnas.93.4.1591}{A fast marching level set method for monotonically advancing fronts},'' \emph{Proceedings of the National Academy of Sciences}, vol.~93, no.~4, pp. 1591--1595, 1996.

\bibitem{PlanningFastMarching}
S.~Garrido, L.~Moreno, M.~Abderrahim, and F.~Martin, ``\href{https://ieeexplore.ieee.org/abstract/document/4058742}{Path Planning for Mobile Robot Navigation using Voronoi Diagram and Fast Marching},'' in \emph{2006 IEEE/RSJ International Conference on Intelligent Robots and Systems}.\hskip 1em plus 0.5em minus 0.4em\relax IEEE, 2006, pp. 2376--2381.

\bibitem{FMT}
L.~Janson, E.~Schmerling, A.~Clark, and M.~Pavone, ``\href{https://journals.sagepub.com/doi/abs/10.1177/0278364915577958}{Fast marching tree: A fast marching sampling-based method for optimal motion planning in many dimensions},'' \emph{The International journal of robotics research}, vol.~34, no.~7, pp. 883--921, 2015.

\bibitem{FastMarchingCubes}
W.~E. Lorensen and H.~E. Cline, ``\href{https://dl.acm.org/doi/abs/10.1145/280811.281026}{Marching cubes: A high resolution 3D surface construction algorithm},'' \emph{ACM siggraph computer graphics}, vol.~21, no.~4, pp. 163--169, 1987.

\bibitem{lei2020analytic}
J.~Lei and K.~Jia, ``\href{http://proceedings.mlr.press/v119/lei20a.html}{Analytic Marching: An Analytic Meshing Solution from Deep Implicit Surface Networks},'' in \emph{International Conference on Machine Learning}.\hskip 1em plus 0.5em minus 0.4em\relax PMLR, 2020, pp. 5789--5798.

\bibitem{xiang2017reachable}
W.~Xiang, H.-D. Tran, and T.~T. Johnson, ``\href{https://arxiv.org/abs/1712.08163}{Reachable Set Computation and Safety Verification for Neural Networks with ReLU Activations},'' \emph{arXiv preprint arXiv:1712.08163}, 2017.

\bibitem{nnv}
H.-D. Tran, D.~M. Lopez, P.~Musau, X.~Yang, L.~V. Nguyen, W.~Xiang, and T.~T. Johnson, ``\href{https://link.springer.com/chapter/10.1007/978-3-030-30942-8_39}{Star-Based Reachability Analysis of Deep Neural Networks},'' in \emph{International Symposium on Formal Methods}.\hskip 1em plus 0.5em minus 0.4em\relax Springer, 2019, pp. 670--686.

\bibitem{yang2020reachability}
X.~Yang, H.-D. Tran, W.~Xiang, and T.~Johnson, ``\href{https://arxiv.org/abs/2003.01226}{Reachability Analysis for Feed-Forward Neural Networks using Face Lattices},'' \emph{arXiv preprint arXiv:2003.01226}, 2020.

\bibitem{nnenum}
S.~Bak, H.-D. Tran, K.~Hobbs, and T.~T. Johnson, ``\href{https://link.springer.com/chapter/10.1007/978-3-030-53288-8_4}{Improved geometric path enumeration for verifying relu neural networks},'' in \emph{Computer Aided Verification: 32nd International Conference, CAV 2020, Los Angeles, CA, USA, July 21--24, 2020, Proceedings, Part I 32}.\hskip 1em plus 0.5em minus 0.4em\relax Springer, 2020, pp. 66--96.

\bibitem{VincentSchwagerICRA21RPM}
J.~A. Vincent and M.~Schwager, ``\href{https://ieeexplore.ieee.org/abstract/document/9561956}{Reachable Polyhedral Marching (RPM): A Safety Verification Algorithm for Robotic Systems with Deep Neural Network Components},'' in \emph{2021 IEEE International Conference on Robotics and Automation (ICRA)}.\hskip 1em plus 0.5em minus 0.4em\relax IEEE, 2021, pp. 9029--9035.

\bibitem{MPT3}
M.~Herceg, M.~Kvasnica, C.~N. Jones, and M.~Morari, ``\href{https://ieeexplore.ieee.org/abstract/document/6669862}{Multi-Parametric Toolbox 3.0},'' in \emph{2013 European control conference (ECC)}.\hskip 1em plus 0.5em minus 0.4em\relax IEEE, 2013, pp. 502--510.

\bibitem{verify_gan}
S.~M. Katz, A.~L. Corso, C.~A. Strong, and M.~J. Kochenderfer, ``\href{https://arc.aiaa.org/doi/abs/10.2514/1.I011071}{ Verification of Image-Based Neural Network Controllers Using Generative Models },'' \emph{Journal of Aerospace Information Systems}, vol.~19, no.~9, pp. 574--584, 2022.

\bibitem{bengio}
G.~F. Montufar, R.~Pascanu, K.~Cho, and Y.~Bengio, ``\href{https://proceedings.neurips.cc/paper_files/paper/2014/hash/109d2dd3608f669ca17920c511c2a41e-Abstract.html}{On the Number of Linear Regions of Deep Neural Networks},'' in \emph{Advances in neural information processing systems}, 2014, pp. 2924--2932.

\bibitem{Hanin_2019}
B.~Hanin, ``\href{https://www.mdpi.com/2227-7390/7/10/992}{Universal Function Approximation by Deep Neural Nets with Bounded Width and ReLU Activations},'' \emph{Mathematics}, vol.~7, no.~10, p. 992, 2019.

\bibitem{understanding}
R.~Arora, A.~Basu, P.~Mianjy, and A.~Mukherjee, ``\href{https://openreview.net/forum?id=B1J_rgWRW}{Understanding Deep Neural Networks with Rectified Linear Units},'' in \emph{International Conference on Learning Representations}, 2018.

\bibitem{fem}
J.~He, L.~Li, J.~Xu, and C.~Zheng, ``\href{https://doc.global-sci.org/uploads/Issue/JCM/shortpdf/v38n3/383_502.pdf}{ReLU Deep Neural Networks and Linear Finite Elements},'' \emph{Journal of Computational Mathematics}, 2019.

\bibitem{reverseengineering}
D.~Rolnick and K.~Kording, ``\href{http://proceedings.mlr.press/v119/rolnick20a.html}{Reverse-engineering deep relu networks},'' in \emph{International Conference on Machine Learning}.\hskip 1em plus 0.5em minus 0.4em\relax PMLR, 2020, pp. 8178--8187.

\bibitem{zico}
J.~Z. Kolter and G.~Manek, ``\href{https://proceedings.neurips.cc/paper/2019/hash/0a4bbceda17a6253386bc9eb45240e25-Abstract.html}{Learning Stable Deep Dynamics Models},'' in \emph{Advances in Neural Information Processing Systems}, 2019, pp. 11\,128--11\,136.

\bibitem{jin2020neural}
W.~Jin, Z.~Wang, Z.~Yang, and S.~Mou, ``\href{https://arxiv.org/abs/2006.08465}{Neural Certificates for Safe Control Policies},'' \emph{arXiv preprint arXiv:2006.08465}, 2020.

\bibitem{art}
X.~Lin, H.~Zhu, R.~Samanta, and S.~Jagannathan, ``\href{https://ieeexplore.ieee.org/abstract/document/9283658}{ART: Abstraction Refinement-Guided Training for Provably Correct Neural Networks},'' in \emph{Formal Methods in Computer-Aided Design}, 2020.

\bibitem{lincolnlab}
S.~Mell, O.~Brown, J.~Goodwin, and S.-H. Son, ``\href{https://arxiv.org/abs/2001.11062}{Safe Predictors for Enforcing Input-Output Specifications},'' \emph{arXiv preprint arXiv:2001.11062}, 2020.

\bibitem{reluplex}
G.~Katz, C.~Barrett, D.~Dill, K.~Julian, and M.~Kochenderfer, ``\href{https://arxiv.org/abs/1702.01135}{Reluplex: An Efficient SMT Solver for Verifying Deep Neural Networks},'' \emph{arXiv preprint arXiv:1702.01135}, 2017.

\bibitem{survey}
C.~Liu, T.~Arnon, C.~Lazarus, C.~Strong, C.~Barrett, M.~J. Kochenderfer, \emph{et~al.}, ``\href{https://www.nowpublishers.com/article/Details/OPT-035}{Algorithms for Verifying Deep Neural Networks},'' \emph{Foundations and Trends{\textregistered} in Optimization}, vol.~4, no. 3-4, pp. 244--404, 2021.

\bibitem{maxsens}
W.~Xiang, H.-D. Tran, and T.~T. Johnson, ``\href{https://ieeexplore.ieee.org/abstract/document/8318388}{Output Reachable Set Estimation and Verification for Multilayer Neural Networks},'' \emph{IEEE transactions on neural networks and learning systems}, vol.~29, no.~11, pp. 5777--5783, 2018.

\bibitem{fastlip}
T.-W. Weng, H.~Zhang, H.~Chen, Z.~Song, C.-J. Hsieh, L.~Daniel, and I.~Dhillon, ``\href{https://proceedings.mlr.press/v80/weng18a.html?utm_source=miragenews&utm_medium=miragenews&utm_campaign=news}{Towards Fast Computation of Certified Robustness for ReLU Networks},'' in \emph{International Conference on Machine Learning (ICML)}, 2018.

\bibitem{reluval}
S.~Wang, K.~Pei, J.~Whitehouse, J.~Yang, and S.~Jana, ``\href{https://www.usenix.org/conference/usenixsecurity18/presentation/wang-shiqi}{Formal Security Analysis of Neural Networks using Symbolic Intervals},'' in \emph{27th USENIX Security Symposium (USENIX Security 18)}, 2018, pp. 1599--1614.

\bibitem{neurify}
S.~\vspace{0mm}Wang, K.~Pei, J.~Whitehouse, J.~Yang, and S.~Jana, ``\href{https://proceedings.neurips.cc/paper/2018/hash/2ecd2bd94734e5dd392d8678bc64cdab-Abstract.html}{Efficient Formal Safety Analysis of Neural Networks},'' in \emph{Advances in Neural Information Processing Systems}, 2018, pp. 6367--6377.

\bibitem{crown}
H.~Zhang, T.-W. Weng, P.-Y. Chen, C.-J. Hsieh, and L.~Daniel, ``\href{https://proceedings.neurips.cc/paper/2018/hash/d04863f100d59b3eb688a11f95b0ae60-Abstract.html}{Efficient Neural Network Robustness Certification with General Activation Functions},'' in \emph{Advances in neural information processing systems}, 2018, pp. 4939--4948.

\bibitem{xiang2020reachable}
W.~Xiang, H.-D. Tran, X.~Yang, and T.~T. Johnson, ``\href{https://ieeexplore.ieee.org/abstract/document/9093970}{Reachable Set Estimation for Neural Network Control Systems: A Simulation-Guided Approach},'' \emph{IEEE Transactions on Neural Networks and Learning Systems}, vol.~32, no.~5, pp. 1821--1830, 2020.

\bibitem{sherlock}
S.~Dutta, S.~Jha, S.~Sanakaranarayanan, and A.~Tiwari, ``\href{https://arxiv.org/abs/1709.09130}{Output Range Analysis for Deep Neural Networks},'' \emph{arXiv preprint arXiv:1709.09130}, 2017.

\bibitem{alessio1}
P.~Kouvaros and A.~Lomuscio, ``\href{https://arxiv.org/abs/1811.11373}{Formal Verification of CNN-based Perception Systems},'' \emph{arXiv preprint arXiv:1811.11373}, 2018.

\bibitem{alessio2}
A.~Lomuscio and L.~Maganti, ``\href{https://arxiv.org/abs/1706.07351}{An approach to reachability analysis for feed-forward ReLU neural networks},'' \emph{arXiv preprint arXiv:1706.07351}, 2017.

\bibitem{proveable}
W.~Ruan, X.~Huang, and M.~Kwiatkowska, ``\href{https://dl.acm.org/doi/abs/10.5555/3304889.3305029}{Reachability analysis of deep neural networks with provable guarantees},'' in \emph{Proceedings of the 27th International Joint Conference on Artificial Intelligence}, 2018, pp. 2651--2659.

\bibitem{verisig}
R.~Ivanov, J.~Weimer, R.~Alur, G.~J. Pappas, and I.~Lee, ``\href{https://dl.acm.org/doi/abs/10.1145/3302504.3311806}{Verisig: verifying safety properties of hybrid systems with neural network controllers},'' in \emph{Proceedings of the 22nd ACM International Conference on Hybrid Systems: Computation and Control}, 2019, pp. 169--178.

\bibitem{abstract}
G.~Singh, T.~Gehr, M.~Püschel, and M.~Vechev, ``\href{https://dl.acm.org/doi/abs/10.1145/3290354}{An abstract domain for certifying neural networks},'' \emph{Proceedings of the ACM on Programming Languages}, vol.~3, pp. 1--30, 01 2019.

\bibitem{AI2}
T.~{Gehr}, M.~{Mirman}, D.~{Drachsler-Cohen}, P.~{Tsankov}, S.~{Chaudhuri}, and M.~{Vechev}, ``\href{https://ieeexplore.ieee.org/abstract/document/8418593}{AI2: Safety and Robustness Certification of Neural Networks with Abstract Interpretation},'' in \emph{2018 IEEE Symposium on Security and Privacy (SP)}, 2018, pp. 3--18.

\bibitem{julian2019reachability}
K.~D. Julian and M.~J. Kochenderfer, ``\href{https://arxiv.org/abs/1903.00520}{A Reachability Method for Verifying Dynamical Systems with Deep Neural Network Controllers},'' \emph{arXiv preprint arXiv:1903.00520}, 2019.

\bibitem{everett}
N.~Rober, M.~Everett, and J.~P. How, ``\href{https://ieeexplore.ieee.org/abstract/document/9992847}{Backward Reachability Analysis for Neural Feedback Loops},'' in \emph{2022 IEEE 61st Conference on Decision and Control (CDC)}.\hskip 1em plus 0.5em minus 0.4em\relax IEEE, 2022, pp. 2897--2904.

\bibitem{overt}
C.~Sidrane, A.~Maleki, A.~Irfan, and M.~J. Kochenderfer, ``\href{https://www.jmlr.org/papers/v23/21-0847.html}{OVERT: An Algorithm for Safety Verification of Neural Network Control Policies for Nonlinear Systems},'' \emph{Journal of Machine Learning Research}, vol.~23, no. 117, pp. 1--45, 2022.

\bibitem{entesari2023automated}
T.~Entesari and M.~Fazlyab, ``\href{https://proceedings.mlr.press/v211/entesari23a.html}{Automated Reachability Analysis of Neural Network-Controlled Systems via Adaptive Polytopes},'' in \emph{Learning for Dynamics and Control Conference}.\hskip 1em plus 0.5em minus 0.4em\relax PMLR, 2023, pp. 407--419.

\bibitem{meng2022learning}
Y.~Meng, D.~Sun, Z.~Qiu, M.~T.~B. Waez, and C.~Fan, ``\href{https://proceedings.mlr.press/v164/meng22a.html}{}learning density distribution of reachable states for autonomous systems,'' in \emph{Conference on Robot Learning}.\hskip 1em plus 0.5em minus 0.4em\relax PMLR, 2022, pp. 124--136.

\bibitem{kochdumper2023open}
N.~Kochdumper, C.~Schilling, M.~Althoff, and S.~Bak, ``\href{https://link.springer.com/chapter/10.1007/978-3-031-33170-1_2}{Open-and closed-loop neural network verification using polynomial zonotopes},'' in \emph{NASA Formal Methods Symposium}.\hskip 1em plus 0.5em minus 0.4em\relax Springer, 2023, pp. 16--36.

\bibitem{rober2023backward}
N.~Rober, S.~M. Katz, C.~Sidrane, E.~Yel, M.~Everett, M.~J. Kochenderfer, and J.~P. How, ``\href{https://ieeexplore.ieee.org/abstract/document/10097878}{Backward Reachability Analysis of Neural Feedback Loops: Techniques for Linear and Nonlinear Systems},'' \emph{IEEE Open Journal of Control Systems}, vol.~2, pp. 108--124, 2023.

\bibitem{zhang2023backward}
Y.~Zhang, H.~Zhang, and X.~Xu, ``\href{https://ieeexplore.ieee.org/abstract/document/10163910}{Backward Reachability Analysis of Neural Feedback Systems Using Hybrid Zonotopes},'' \emph{IEEE Control Systems Letters}, vol.~7, pp. 2779--2784, 2023.

\bibitem{robinson2020dissecting}
H.~Robinson, A.~Rasheed, and O.~San, ``\href{https://arxiv.org/abs/1910.03879}{Dissecting Deep Neural Networks},'' \emph{arXiv preprint arXiv:1910.03879}, 2019.

\bibitem{hanin2019complexity}
B.~Hanin and D.~Rolnick, ``\href{https://proceedings.mlr.press/v97/hanin19a}{Complexity of Linear Regions in Deep Networks},'' in \emph{International Conference on Machine Learning}, 2019, pp. 2596--2604.

\bibitem{xu2022traversing}
S.~Xu, J.~Vaughan, J.~Chen, A.~Zhang, and A.~Sudjianto, ``\href{https://openreview.net/forum?id=EQjwT2-Vaba}{Traversing the Local Polytopes of ReLU Neural Networks },'' in \emph{The AAAI-22 Workshop on Adversarial Machine Learning and Beyond}, 2022.

\bibitem{eMPC_book}
F.~Borrelli, A.~Bemporad, and M.~Morari, \emph{\href{https://www.cambridge.org/highereducation/books/predictive-control-for-linear-and-hybrid-systems/EF618BD7AFAF4D04B2044A0FD03D885A\#overview}{Predictive Control for Linear and Hybrid Systems}}.\hskip 1em plus 0.5em minus 0.4em\relax Cambridge University Press, 2017.

\bibitem{ellip_roa}
H.~Yin, P.~Seiler, and M.~Arcak, ``\href{https://ieeexplore.ieee.org/abstract/document/9388885}{Stability Analysis using Quadratic Constraints for Systems with Neural Network Controllers},'' \emph{IEEE Transactions on Automatic Control}, pp. 1--1, 2021.

\bibitem{richards2018lyapunov}
S.~M. Richards, F.~Berkenkamp, and A.~Krause, ``\href{https://proceedings.mlr.press/v87/richards18a.html}{https://proceedings.mlr.press/v87/richards18a.html},'' in \emph{Conference on Robot Learning}.\hskip 1em plus 0.5em minus 0.4em\relax PMLR, 2018, pp. 466--476.

\bibitem{chen2021learning}
S.~Chen, M.~Fazlyab, M.~Morari, G.~J. Pappas, and V.~M. Preciado, ``\href{https://ieeexplore.ieee.org/abstract/document/9682880}{Learning Region of Attraction for Nonlinear Systems},'' in \emph{2021 60th IEEE Conference on Decision and Control (CDC)}.\hskip 1em plus 0.5em minus 0.4em\relax IEEE, 2021, pp. 6477--6484.

\bibitem{dai2021lyapunov}
H.~Dai, B.~Landry, L.~Yang, M.~Pavone, and R.~Tedrake, ``\href{https://arxiv.org/abs/2109.14152}{Lyapunov-stable neural-network control},'' \emph{arXiv preprint arXiv:2109.14152}, 2021.

\bibitem{control_inv_changliu}
T.~Wei and C.~Liu, ``\href{https://proceedings.mlr.press/v168/wei22a.html}{Safe Control with Neural Network Dynamic Models},'' in \emph{Learning for Dynamics and Control Conference}.\hskip 1em plus 0.5em minus 0.4em\relax PMLR, 2022, pp. 739--750.

\bibitem{jouret2023safety}
L.~Jouret, A.~Saoud, and S.~Olaru, ``\href{https://ieeexplore.ieee.org/abstract/document/10354438}{Safety verification of Neural-Network-based controllers: a set invariance approach},'' \emph{IEEE Control Systems Letters}, 2023.

\bibitem{biswas}
P.~Biswas, P.~Grieder, J.~L{\"o}fberg, and M.~Morari, ``\href{https://www.sciencedirect.com/science/article/pii/S1474667016363443}{A survey on stability analysis of discrete-time piecewise affine systems},'' \emph{IFAC Proceedings Volumes}, vol.~38, no.~1, pp. 283--294, 2005.

\bibitem{pwa_lyap}
M.~Rubagotti, L.~Zaccarian, and A.~Bemporad, ``\href{https://www.tandfonline.com/doi/full/10.1080/00207179.2015.1108456}{A Lyapunov Method for Stability Analysis of Piecewise-Affine Systems Over Non-Invariant Domains},'' \emph{International Journal of Control}, vol.~89, pp. 1--25, 10 2015.

\bibitem{lmi_lyap}
A.~Hassibi and S.~Boyd, ``\href{https://ieeexplore.ieee.org/abstract/document/703296}{Quadratic stabilization and control of piecewise-linear systems},'' in \emph{Proceedings of the 1998 American Control Conference. ACC (IEEE Cat. No.98CH36207)}, vol.~6, 1998, pp. 3659--3664 vol.6.

\bibitem{baldi2018reachable}
S.~Baldi and W.~Xiang, ``\href{https://ieeexplore.ieee.org/document/7799387}{Reachable set estimation for switched linear systems with dwell-time switching},'' \emph{Nonlinear Analysis: Hybrid Systems}, vol.~29, pp. 20--33, 2018.

\bibitem{pwa_lygeros}
S.~Rakovic, E.~Kerrigan, D.~Mayne, and J.~Lygeros, ``\href{https://ieeexplore.ieee.org/abstract/document/1618830}{Reachability analysis of discrete-time systems with disturbances},'' \emph{IEEE Transactions on Automatic Control}, vol.~51, no.~4, pp. 546--561, 2006.

\bibitem{hanin2019deep}
B.~\vspace{0mm}Hanin and D.~Rolnick, ``\href{https://proceedings.neurips.cc/paper_files/paper/2019/hash/9766527f2b5d3e95d4a733fcfb77bd7e-Abstract.html}{Deep ReLU Networks Have Surprisingly Few Activation Patterns},'' in \emph{Advances in Neural Information Processing Systems}, 2019, pp. 361--370.

\bibitem{lattice_orig}
R.~H. Wilkinson, ``\href{https://ieeexplore.ieee.org/abstract/document/4037806}{A Method of Generating Functions of Several Variables Using Analog Diode Logic},'' \emph{IEEE Transactions on Electronic Computers}, vol. EC-12, no.~2, pp. 112--129, 1963.

\bibitem{lattice_irred}
J.~Xu, T.~J. van~den Boom, and B.~De~Schutter, ``\href{https://ieeexplore.ieee.org/abstract/document/7040078}{Irredundant lattice piecewise affine representations and their applications in explicit model predictive control},'' in \emph{53rd IEEE Conference on Decision and Control}, 2014, pp. 4416--4421.

\bibitem{canonical}
L.~Chua and A.-C. Deng, ``\href{https://ieeexplore.ieee.org/abstract/document/1705}{Canonical piecewise-linear representation},'' \emph{IEEE Transactions on Circuits and Systems}, vol.~35, no.~1, pp. 101--111, 1988.

\bibitem{pwa_comp}
T.~Kevenaar and D.~Leenaerts, ``\href{https://ieeexplore.ieee.org/abstract/document/207720}{A comparison of piecewise-linear model descriptions},'' \emph{IEEE Transactions on Circuits and Systems I: Fundamental Theory and Applications}, vol.~39, no.~12, pp. 996--1004, 1992.

\bibitem{morari}
R.~Suard, J.~Lofberg, P.~Grieder, M.~Kvasnica, and M.~Morari, ``\href{https://ieeexplore.ieee.org/abstract/document/1429297}{Efficient computation of controller partitions in multi-parametric programming},'' in \emph{2004 43rd IEEE Conference on Decision and Control (CDC)(IEEE Cat. No. 04CH37601)}, vol.~4.\hskip 1em plus 0.5em minus 0.4em\relax IEEE, 2004, pp. 3643--3648.

\bibitem{cdd_manual}
K.~Fukuda, ``\href{https://people.inf.ethz.ch/fukudak/cdd_home/cddlibman2021.pdf}{Cddlib reference manual},'' \emph{Report version 093a, McGill University, Montr{\'e}al, Quebec, Canada}, 2003.

\bibitem{mpLP}
A.~Mar{\'e}chal, D.~Monniaux, and M.~P{\'e}rin, ``\href{https://link.springer.com/chapter/10.1007/978-3-319-66706-5_11}{Scalable Minimizing-Operators on Polyhedra via Parametric Linear Programming},'' in \emph{International Static Analysis Symposium}.\hskip 1em plus 0.5em minus 0.4em\relax Springer, 2017, pp. 212--231.

\bibitem{cdd}
\BIBentryALTinterwordspacing
B.~Legat, R.~Deits, M.~Forets, D.~Oyama, S.~Timme, F.~Pacaud, S.~Guadalupe, M.~Besançon, J.~TagBot, and E.~Saba, ``Juliapolyhedra/cddlib.jl: v0.6.1,'' Mar. 2020. [Online]. Available: \url{https://doi.org/10.5281/zenodo.3733590}
\BIBentrySTDinterwordspacing

\bibitem{hennet1995discrete}
J.-C. Hennet, ``\href{https://citeseerx.ist.psu.edu/document?repid=rep1&type=pdf&doi=6889155edab0024bec6acab2b662f3f0ee2d6177}{Discrete time constrained linear systems},'' \emph{Control and dynamic systems}, vol.~71, pp. 157--214, 1995.

\bibitem{homeo2}
W.~C. Rheinboldt and J.~S. Vandergraft, ``\href{https://epubs.siam.org/doi/abs/10.1137/0129056}{On Piecewise Affine Mappings in $R^n$},'' \emph{SIAM Journal on Applied Mathematics}, vol.~29, no.~4, pp. 680--689, 1975.

\bibitem{papamakarios2021normalizing}
G.~Papamakarios, E.~Nalisnick, D.~J. Rezende, S.~Mohamed, and B.~Lakshminarayanan, ``\href{https://www.jmlr.org/papers/v22/19-1028.html}{Normalizing flows for probabilistic modeling and inference},'' \emph{Journal of Machine Learning Research}, vol.~22, no.~57, pp. 1--64, 2021.

\bibitem{DunningHuchetteLubin2017}
I.~Dunning, J.~Huchette, and M.~Lubin, ``\href{https://epubs.siam.org/doi/abs/10.1137/15M1020575}{JuMP: A Modeling Language for Mathematical Optimization},'' \emph{SIAM Review}, vol.~59, no.~2, pp. 295--320, 2017.

\end{thebibliography}

\appendix 
\label{Sec:Appendix}

\subsection{Proofs}
In this section we provide a proof for Theorem \ref{thm: activation_flipping}. 
Similar arguments are used to characterize the change in Jacobians of PWA components in \cite{pwa_comp} and the references therein.
\textcolor{black}{We begin with a proof for the first claim of Theorem \ref{thm: activation_flipping} with Lemma \ref{lemma:stationary}.
We then prove the second claim of Theorem \ref{thm: activation_flipping} with Lemma \ref{lemma:flips}, with additional support from Lemma \ref{lemma:affine}.}



\subsubsection{Proving Claim 1 of Theorem \ref{thm: activation_flipping}}
\begin{lemma}[Stationary Neurons] \label{lemma:stationary}
If $[\ba_{k,ij}^\top\quad  -b_{k,ij}] \neq \alpha[\ba_\eta^\top\quad -b_\eta]$ for some $\alpha \in \{0,1\}$, then the activation of neuron $ij$ does not change, $\lambda_{k,ij} = \lambda_{k',ij}$.
\end{lemma}
\begin{proof}
\textcolor{black}{We prove this lemma by showing that, under the conditions of the lemma, the constraint induced by the neuron under its activation in $P_k$ remains valid in $P_{k'}$.}

If $[\ba_{k,ij}^\top\quad  -b_{k,ij}] \neq \alpha[\ba_\eta^\top\quad -b_\eta]$ for some $\alpha \in \{0,1\}$, then we know that the constraint $\ba_{k,ij}^\top \bx \le  b_{k,ij}$ is not active on the boundary between $P_k$ and $P_{k'}$. Formally, $\forall \bx \in \interior(P_k \cap P_{k'})$
\begin{align}
    \ba_{k,ij}^{\top}(\bx) \bx < b_{k,ij}(\bx). \label{eq:App_T3}
\end{align}
The inequality must be strict because if the relationship held with equality $\forall \bx \in \interior(P_k \cap P_{k'})$ then we would have $[\ba_{k,ij}^\top\quad  -b_{k,ij}] = \alpha[\ba_\eta^\top\quad -b_\eta]$ for some $\alpha \in \{0,1\}$ (by definition of the neighbor constraint), leading to a contradiction.

Furthermore, if equality held on a strict subset of $\interior(P_k \cap P_{k'})$, then the shared face of $P_k$ and $P_{k'}$ would not be identified by $\ba_\eta^\top \bx = b_\eta$, but instead by both $\ba_\eta^\top \bx = b_\eta$ and  $\ba_{k,ij}^\top \bx = b_{k,ij}$, another contradiction. 

We now know the constraint is not active due to the strict inequality in (\ref{eq:App_T3}).
In addition, note that $\ba_{k,ij}^{\top}(\bx) \bx - b_{k,ij}(\bx)$ is a continuous function of $\bx$. 
From these facts we know $\exists \bm{\delta}$ such that $\bx + \bm{\delta} \in \interior(P_{k'})$ and the inequality remains strict,
\begin{align}
    \ba_{k,ij}^{\top}(\bx + \bm{\delta}) (\bx + \bm{\delta}) < b_{k,ij}(\bx + \bm{\delta}). \label{eq:App_T12}
\end{align}
Therefore, since the constraint associated with such a neuron $ij$ is valid on the interior of both $P_k$ and $P_{k'}$, the neuron activation does not change.

To finish, we note that the case of $\alpha = -1$ is impossible because it would imply that the dimension of $\text{int}(P_k)$ is less than the ambient dimension (contradicting Definition \ref{def:tessellation}).
\end{proof}

\subsubsection{Proving Claim 2 of Theorem \ref{thm: activation_flipping}}
\begin{lemma}[Effects of Activation Flips] \label{lemma:flips}
If \newline $[\ba_{k,ij}^\top\quad  -b_{k,ij}] = \alpha[\ba_\eta^\top\quad -b_\eta]$ for some $\alpha \in \{0,1\}$, then $[\ba_{k',ij}^\top\quad  -b_{k',ij}] = \alpha'[\ba_\eta^\top\quad -b_\eta]$ for some $\alpha' \in \{-1,0\}$.
\end{lemma}
\begin{proof}
\textcolor{black}{We prove this lemma by first showing that a weaker condition (dependent on $\bx$) holds for all points on the boundary between the regions.
Then, by analyzing this weaker condition for $n$ carefully chosen points on the boundary, we show that the dependence on $\bx$ can be removed and we arrive at the desired result.}

Consider the affine functions defined by neuron $ij$ in regions $P_k$ and $P_{k'}$,
\begin{subequations}
\begin{align}
    \Bar{\ba}_{k,ij}^\top \bx -\Bar{b}_{k,ij} \quad &\forall \bx \in P_k \\
    \Bar{\ba}_{k',ij}^\top \bx -\Bar{b}_{k',ij} \quad &\forall \bx \in P_{k'}
\end{align}
\end{subequations}
(note that the overbar accent indicates the parameters are unnormalized as in (\ref{Eq:HalfspaceDef}). Since ReLU networks are continuous functions, $\forall \bx \in P_k \cap P_{k'}$ we know
\begin{subequations}
\begin{align}
    \Bar{\ba}_{k,ij}^\top \bx -\Bar{b}_{k,ij} &= \Bar{\ba}_{k',ij}^\top \bx -\Bar{b}_{k',ij} \\
    \implies \begin{bmatrix} \multicolumn{2}{c}{\Bar{\ba}_{k',ij}^\top} & -\Bar{b}_{k',ij} \end{bmatrix} \begin{bmatrix} \bx \\ 1 \end{bmatrix} &= 
    \begin{bmatrix} \multicolumn{2}{c}{\Bar{\ba}_{k,ij}^\top} & -\Bar{b}_{k,ij} \end{bmatrix} \begin{bmatrix} \bx \\ 1 \end{bmatrix} \\
    \implies \begin{bmatrix} \multicolumn{2}{c}{\Bar{\ba}_{k',ij}^\top} & -\Bar{b}_{k',ij} \end{bmatrix} \begin{bmatrix} \bx \\ 1 \end{bmatrix} &= 
    \beta \begin{bmatrix}  \multicolumn{2}{c}{\ba_{\eta}^\top} & -b_{\eta} \end{bmatrix} \begin{bmatrix} \bx \\ 1 \end{bmatrix}
\end{align}
\end{subequations}
for some $\beta \ge 0$. Thus, we know the normalized constraint $\ba_{k',ij}^\top \bx \le b_{k',ij}$ must satisfy
\begin{align}
    \begin{bmatrix} \multicolumn{2}{c}{\ba_{k',ij}^\top} & -b_{k',ij} \end{bmatrix}  \begin{bmatrix} \bx \\ 1 \end{bmatrix} &=  
    \alpha' \begin{bmatrix} \ba^\top_\eta & -b_\eta \end{bmatrix} 
    \begin{bmatrix} \bx \\ 1 \end{bmatrix} \label{eq:x_depend}
\end{align}
for some $\alpha' \in \{-1,0,1\}$ and $\forall \bx \in P_k \cap P_{k'}$. We next seek to remove this relationship's dependence on $\bx$.

We know $P_k \cap P_{k'}$ is a subset of the $n-1$-dimensional affine span $\{\bx \mid \ba_\eta^\top \bx = b_\eta \}$. Since $\vert P_k \cap P_{k'} \vert_{n-1} \neq 0$ there exist $\bx_1, \ldots, \bx_{n} \in P_k \cap P_{k'}$ that define the vertex points of an $n-1$-dimensional simplex on $P_k \cap P_{k'}$. The vertices of a simplex are known to be affinely independent. Therefore, by Lemma \ref{lemma:affine}, the vectors
\begin{align}
    \begin{bmatrix} \bx_1 \\ 1 \end{bmatrix}, \ldots, \begin{bmatrix} \bx_{n} \\ 1 \end{bmatrix} \in \mathbb{R}^{n+1}
\end{align}
are linearly independent. Next, consider the matrix formed by these vectors
\begin{align}
    \mathbf{X} = \begin{bmatrix} \bx_1^\top & 1 \\ \vdots & \vdots \\ \bx_{n}^\top & 1 \end{bmatrix} \in \mathbb{R}^{n \times n+1}.
\end{align}
From (\ref{eq:x_depend}), and since $\ba_\eta^\top \bx_i - b_\eta = 0 \quad \forall i \in 1,\ldots,n$,
\begin{subequations}
\begin{align}
    \mathbf{X} \begin{bmatrix} \ba_{k',ij} \\ -b_{k',ij} \end{bmatrix}  &= \alpha' \mathbf{X}  \begin{bmatrix} \ba_\eta \\ -b_\eta \end{bmatrix}\\
    \mathbf{X} \begin{bmatrix} \ba_{k',ij} \\ -b_{k',ij} \end{bmatrix} &= \mathbf{0}.
\end{align}
\end{subequations}

By the rank-nullity theorem, nullity$(\mathbf{X}) = 1$ (rank is $n$ by construction and $\mathbf{X} \in \mathbb{R}^{n \times n+1}$). Furthermore, since both $[\ba_{k',ij}^\top \quad -b_{k',ij}]^\top$ and $[\ba_\eta^\top \quad -b_\eta]^\top$ lie in a single-dimensional nullspace, they must be scalar multiples of each other.
Thus, for some $\alpha' \in \{-1,0,1\}$,
\begin{subequations}
    \begin{align}
    [\ba_{k',ij}^\top \quad -b_{k',ij}]^\top, [\ba_\eta^\top \quad -b_\eta]^\top \in \text{null}(\mathbf{X}) \\
    \implies [\ba_{k',ij}^\top\quad  -b_{k',ij}] = \alpha'[\ba_\eta^\top\quad -b_\eta].
\end{align}
\end{subequations}

Lastly, $\begin{bmatrix} \multicolumn{2}{c}{\ba_{k',ij}^\top} & -b_{k',ij} \end{bmatrix} = \begin{bmatrix} \ba^\top_\eta & -b_\eta \end{bmatrix}$ cannot hold as it leaves $P_{k'}$ lower than the ambient dimension, so $\alpha' \neq 1$. Thus, we arrive at the desired result.
\end{proof}

\begin{lemma}[Affine Independence to Linear Independence]\label{lemma:affine}
Given affinely independent vectors $\bx_1,\ldots,\bx_k \in \mathbb{R}^n$, $[\bx_1^\top \quad 1 ]^\top, \ldots, [ \bx_k^\top \quad 1 ]^\top \in \mathbb{R}^{n+1}$ are linearly independent.
\end{lemma}
\begin{proof}
The vectors $[\bx_1^\top \quad 1 ]^\top, \ldots, [ \bx_k^\top \quad 1 ]^\top$ are linearly independent if and only if
\begin{subequations}
\begin{align}
    \lambda_1\begin{bmatrix} \bx_1 \\ 1 \end{bmatrix} + \lambda_2\begin{bmatrix} \bx_2 \\ 1 \end{bmatrix} + \ldots + \lambda_k\begin{bmatrix} \bx_{k} \\ 1 \end{bmatrix} &= \mathbf{0} \\
    \implies \lambda_1 = \lambda_2 = \ldots = \lambda_k &= 0.
\end{align}
\end{subequations}
This condition is equivalent to
\begin{subequations}
\begin{align}
    \begin{rcases}
    \lambda_1 \bx_1 + \lambda_2 \bx_2 + \ldots + \lambda_k \bx_k &= \mathbf{0} \\
    \lambda_1 + \lambda_2 + \ldots + \lambda_k &= 0
  \end{rcases} \\
  \implies \lambda_1 = \lambda_2 = \ldots = \lambda_k &= 0,
\end{align}
\end{subequations}
i.e. $\bx_1,\ldots,\bx_k$ are affinely independent.
\end{proof}

\end{document}